\newcommand\longversion[1]{#1}
\newcommand\shortversion[1]{}
\newcommand\aaaiversion[1]{}

\longversion{
\documentclass[preprint,12pt]{elsarticle}
\usepackage{graphics,amsthm,fullpage}

\journal{Theoretical Computer Science}
}

\usepackage{amssymb,amsmath}
\usepackage{perpage} 
\usepackage{array}
\usepackage[usenames,dvipsnames]{color}
\usepackage[table,xcdraw]{xcolor}
\usepackage[colorlinks=true,linkcolor=blue,citecolor=Brown]{hyperref}

\usepackage{enumitem}

\usepackage{makecell}
\usepackage{eulervm}

\newcommand{\defproblem}[3]{
\begin{center}
\noindent\fbox{

  \begin{minipage}{\textwidth}
  \begin{tabular*}{\textwidth}{@{\extracolsep{\fill}}l} \textsc{\underline{#1}} \\ \end{tabular*}\vspace{1ex}
  {\bf{Input:}} #2  \\
  {\bf{Question:}} #3
  \end{minipage}
 
  }
\end{center}
}

\usepackage{multirow}

\MakePerPage{footnote} 
\newcommand{\ignore}[1]{}

\shortversion{\newenvironment{proof}{\noindent{\em Proof:}}{ \hfill $\square$\\ }}

\usepackage{amsopn}

\shortversion{
\allowdisplaybreaks
\allowbreak

\usepackage{titlesec}
\titlespacing{\section}{0pt}{0pt}{0pt}
\titlespacing{\subsection}{0pt}{0pt}{0pt}
\titlespacing{\subsubsection}{0pt}{0pt}{0pt}

\setlength{\belowdisplayskip}{0pt} 
\setlength{\abovedisplayskip}{0pt} 
}

\longversion{
\usepackage{charter}
}

\definecolor{arylideyellow}{rgb}{0.91, 0.84, 0.42}

\usepackage{nicefrac}

\newcommand{\nfrac}{\nicefrac}

\usepackage{todonotes,xspace}
\usepackage{float}

\renewcommand{\ge}{\geqslant}
\renewcommand{\le}{\leqslant}

\newcommand{\pr}{\prime}
\newcommand{\suc}{\ensuremath{\succ}\xspace}

\newcommand{\YES}{\textsc{Yes}\xspace}
\newcommand{\NO}{\textsc{No}\xspace}

\newcommand{\PS}{\textsc{Permutation sum}\xspace}

\newcommand{\NP}{\ensuremath{\mathsf{NP}}\xspace}

\newcommand{\NPC}{\ensuremath{\mathsf{NP}}-complete\xspace}

\newcommand{\Pshort}{\ensuremath{\mathsf{P}}\xspace}
\newcommand{\NPshort}{\ensuremath{\mathsf{NP}}\xspace}

\usepackage{xspace}

\newcommand{\CC}{\ensuremath{\mathcal C}\xspace}

\newcommand{\PP}{\ensuremath{\mathcal P}\xspace}

\newcommand{\VV}{\ensuremath{\mathcal V}\xspace}

\newcommand{\XX}{\ensuremath{\mathcal X}\xspace}

\newcommand{\xxx}{\ensuremath{\mathfrak x}\xspace}

\newcommand{\el}{\ensuremath{\ell}\xspace}

\usepackage{cleveref}

\newtheorem{proposition}{\bf Proposition}
\newtheorem{theorem}{\bf Theorem}
\newtheorem{lemma}{\bf Lemma}
\newtheorem{lemma*}{\bf Lemma$^{\star}$}
\newtheorem{corollary}{\bf Corollary}
\newtheorem{definition}{\bf Definition}
\newtheorem{observation}{\bf Observation}

\crefname{theorem}{theorem}{ Theorem}
\crefname{lemma}{lemma}{\bf Lemma}
\crefname{lemma*}{lemma}{\bf Lemma}
\crefname{corollary}{corollary}{\bf Corollary}
\crefname{proposition}{proposition}{\bf Proposition}
\crefname{definition}{definition}{\bf Definition}
\crefname{observation}{observation}{\bf Observation}
\crefname{table}{table}{\bf Table}

\title{Frugal Bribery in Voting}
\shortversion{
\author{AAAI Submission Number: 1421}
}
\longversion{
\author{Palash Dey$^\star$, Neeldhara Misra$^\dagger$, and Y. Narahari$^\S$\\ \texttt{palash@csa.iisc.ernet.in, mail@neeldhara.com, hari@csa.iisc.ernet.in}}

\address{$^\star$Tata Institute of Fundamental Research, Mumbai\\
$^\dagger$Indian Institute of Technology, Gandhinagar\\
$^\S$Indian Institute of Science, Bangalore}
}

\begin{document}
\sloppy


\begin{frontmatter}
\begin{abstract} Bribery in elections is an important problem in computational
social choice theory. We introduce and study two important special cases of the classical
\textsc{\$Bribery} problem, namely, \textsc{Frugal-bribery} and \textsc{Frugal-\$bribery} where the briber is frugal in nature. By this, we mean that the briber is 
only able to influence voters who benefit from the suggestion of the briber.
More formally, a voter is {\em vulnerable} if the outcome of the election  
improves according to her own preference when she accepts the suggestion of the 
briber. In the \textsc{Frugal-bribery} problem, the goal is to make a certain
candidate win the election by changing {\em only} the vulnerable votes. In the
\textsc{Frugal-\$bribery} problem, the vulnerable votes have prices and the
goal is to make a certain candidate win the election by changing only the 
vulnerable votes, subject to a budget constraint. We further
formulate two natural variants of the \textsc{Frugal-\$bribery} problem namely
\textsc{Uniform-frugal-\$bribery} and \textsc{Nonuniform-frugal-\$bribery}
where the prices of the vulnerable votes are, respectively, all the same or different. The \textsc{Frugal-bribery} problem turns out to be a special case of sophisticated {\sc \$Bribery} as well as {\sc Swap-bribery} problems. Whereas the \textsc{Frugal-\$bribery} problem turns out to be a special case of the {\sc \$Bribery} problem.

We show that the \textsc{Frugal-bribery} problem is polynomial time solvable for the $k$-approval, $k$-veto, and plurality with run off voting rules for unweighted elections. These results establish success in finding practically appealing as well as polynomial time solvable special cases of the sophisticated {\sc \$Bribery} and {\sc Swap-bribery} problems. On the other hand, we show that the \textsc{Frugal-bribery} problem is \NPC for the Borda voting rule and the \textsc{Frugal-\$bribery} problem is \NPC for most of the voting rules studied here barring the plurality and the veto voting rules for unweighted elections. Our hardness results of the \textsc{Frugal-bribery} and the \textsc{Frugal-\$bribery} problems thus subsumes and strengthens the hardness results of the {\sc \$Bribery} problem from the literature. For the weighted elections, we show that the \textsc{Frugal-bribery} problem is \NPC for all the voting rules studied here except the plurality voting rule even when the number of candidates is as low as $3$ (for the STV and the plurality with run off voting rules) or $4$ (for the maximin, the Copeland$^\alpha$ with $\alpha\in[0,1)$, and the simplified Bucklin voting rules). In our view, the fact that the simplest \textsc{Frugal-bribery} problem becomes computationally intractable for many important voting rules (except the plurality voting rule) even with very few candidates is surprising as well as interesting.


\end{abstract}


\end{frontmatter}
\section{Introduction}
In a typical voting scenario, we have a set of candidates and a set of voters 
reporting their \emph{preferences or votes} which are complete rankings over the candidates. 
A \emph{voting rule} is a procedure that, given a collection of votes, chooses one candidate as the winner. A set of votes over a set of candidates along with a voting rule is called an election.

Activities that try to influence voter opinions, in favor of specific candidates, are very common during the time that an election is in progress. For example, in a political election, candidates often conduct elaborate campaigns to promote themselves among a general or targeted audience. Similarly, it is not uncommon for people to protest against, or rally for, a national committee or court that is in the process of approving a particular policy. An extreme illustration of this phenomenon is \emph{bribery} --- here, the candidates may create financial incentives to sway the voters. Of course, the process of influencing voters may involve costs even without the bribery aspect; for instance, a typical political campaign or rally entails considerable expenditure. 

All situations involving a systematic attempt to influence voters usually have the following aspects: an external agent, a candidate that the agent would like to be the winner, a budget constraint, a cost model for a change of vote, and knowledge of the existing election. The formal computational problem that arises from these inputs is the following: is it possible to make a distinguished candidate win the election in question by incurring a cost that is within the budget? This question, with origins in Faliszewski et al.~\cite{faliszewski2006complexity,faliszewski2009hard,faliszewski2009llull}, has been subsequently studied intensely in computational social choice literature\shortversion{~\cite{faliszewski2008nonuniform,elkind2009swap,baumeister2012campaigns,pini2013bribery,dorn2014hardness,mattei2012bribery,MatteiPRV13,erdelyi2009complexity,erdelyi2014bribery,faliszewski2014complexity,xia2012computing,dorn2012multivariate,bredereck2014prices,FaliszewskiHH11,SchlotterFE11}}. 
In particular, bribery has been studied under various cost models, for example, uniform price per 
vote which is known as \textsc{\$Bribery}~\cite{faliszewski2006complexity}, nonuniform price per vote~\cite{faliszewski2008nonuniform},  nonuniform price per shift of the distinguished candidate per vote which is called \textsc{Shift bribery}, nonuniform price per swap of candidates 
per vote which is called \textsc{Swap bribery}~\cite{elkind2009swap}. A closely related problem known as campaigning has been studied for various vote models, for example, truncated ballots~\cite{baumeister2012campaigns}, 
soft constraints~\cite{pini2013bribery}, CP-nets~\cite{dorn2014hardness}, combinatorial domains~\cite{mattei2012bribery} and 
probabilistic lobbying~\cite{erdelyi2009complexity}. 
The bribery problem has also been studied under 
voting rule uncertainty~\cite{erdelyi2014bribery}. Faliszewski et al.~\cite{faliszewski2014complexity} 
study the complexity of bribery in simplified Bucklin and Fallback voting rules. Xia~\cite{xia2012computing} studies destructive bribery, 
where the goal of the briber is to change the winner by changing minimum number of votes. 
Dorn et al.~\cite{dorn2012multivariate} studies the 
parameterized complexity of the \textsc{Swap Bribery} problem and Bredereck et al.~\cite{bredereck2014prices} explores the parameterized 
complexity of the \textsc{Shift Bribery} problem for a wide range of parameters.
We recall again that the costs and the budgets involved in all the bribery problems above need not necessarily 
correspond to actual money traded between voters and candidates. They may correspond to any cost in general, for example, the amount of effort or time that the briber needs to spend for each voter.

\subsection{Motivation}
In this work, we propose an effective cost model for the bribery problem. Even the most general cost models that have been studied in the literature fix absolute costs per voter-candidate combination, with no specific consideration to the voters' opinions about the current winner and the distinguished candidate whom the briber wants to be the winner. In our proposed model, a change of vote is relatively easier to effect if the change causes an outcome that the voter would find desirable. Indeed, if the currently winning candidate is, say, $a$, and a voter is (truthfully) promised that by changing her vote from $c \succ d \succ a \succ b$ to $d \succ b \succ c \succ a$, the winner of the election would change from $a$ to $d$, then this is a change that the voter is likely to be  happy to make. While the change does not make her most favorite candidate win the election, it does improve the result from her point of view. Thus, given the circumstances (namely that of her least favorite candidate winning the election), the altered vote serves the voter better than the original one. 

We believe this perspective of voter influence is an important one to study. The cost of a change of vote is proportional to the nature of the outcome that the change promises --- the cost is low or nil if the change results in a better outcome with respect to the voter's original ranking, and high or infinity otherwise. A frugal agent only approaches voters of the former category, thus being able to effectively bribe with minimal or no cost. Indeed the behavior of agents in real life is often frugal. For example, consider campaigners in favor of a relatively smaller party in a political election. They may actually target only vulnerable voters due to lack of human and other resources they have at their disposal.

More formally, let $c$ be the winner of an election and $p$ (other than $c$) the candidate whom the briber 
wishes to make the winner of the election. Now the voters who prefer $c$ to $p$ will be reluctant to change their votes, and we call these votes {\em non-vulnerable with respect to $p$} --- we do not allow these votes to be changed by the briber, which justifies the {\em frugal} nature of the briber. On the other hand, if a voter prefers $p$ to $c$, then it may be very easy to convince her to change her vote if doing so makes $p$ win the election. We name these votes {\em vulnerable with respect to $p$}. When the candidate $p$ is clear from the context, we simply call these votes non-vulnerable and vulnerable, respectively. 

The computational problem is to determine whether there is a way to make a candidate $p$ win the election by changing \emph{only} those votes that are vulnerable with respect to $p$. We call this problem \textsc{Frugal-bribery}. Note that there is no cost involved in the \textsc{Frugal-bribery} problem --- the briber does not incur any cost to change the votes of the vulnerable votes. We also extend this basic model to a more general setting where each vulnerable vote has a certain nonnegative integer price which may correspond to the effort involved in approaching these voters and convincing them to change their votes. We also allow for the specification of a budget constraint, which can be used to enforce auxiliary constraints. This leads us to define the \textsc{Frugal-\$bribery} problem, where we are required to find a subset of vulnerable votes with a total cost that is within a given budget, such that these votes can be changed in some way to make the candidate $p$ win the election. Note that the \textsc{Frugal-\$bribery} problem can be either uniform or nonuniform depending on whether the prices of the vulnerable votes are all identical or different. If not mentioned otherwise, the prices of the vulnerable votes will be assumed to be nonuniform. 
We remind that the briber is not allowed to change the non-vulnerable votes in both the \textsc{Frugal-bribery} and the \textsc{Frugal-\$bribery} problems.

\subsection{Contributions}

Our primary contribution in this paper is to formulate and study two important and natural models of bribery which turn out to be special cases of the well studied \textsc{\$Bribery} problem in elections. Indeed, the \textsc{Frugal-\$bribery} problem and, more importantly, the \textsc{Frugal-bribery} problem are very restricted yet practically appealing cases of the \textsc{\$Bribery} problem.


\subsubsection*{Our Results for Unweighted Elections} 

We have the following polynomial time algorithms for unweighted elections. These results show that the \textsc{Frugal-bribery} problem is computationally tractable for some voting rules for which the \textsc{\$Bribery} problem is \NPC as observed for the $k$-approval with $k\ge 3$~\cite{Lin11}, simplified Bucklin~\cite{FaliszewskiRRS15}. We summarize the results in \Cref{tbl:frugal_summary}.

 \begin{itemize}
  \item The \textsc{Frugal-bribery} problem is in \Pshort{} for the $k$-approval, simplified Bucklin, and plurality with runoff voting rules. Also, the \textsc{Frugal-\$bribery} problem is in \Pshort{} for the plurality and veto voting rules. 
  \item The \textsc{Frugal-\$bribery} problem is in \Pshort{} for the $k$-approval, simplified Bucklin, and plurality with runoff voting rules when the budget is a constant [\Cref{thm:frugalKappP}].
 \end{itemize}
 
 We have the following intractability results for the \textsc{Frugal-bribery} problem and the \textsc{Frugal-\$bribery} problem for unweighted elections. Our hardness results of the \textsc{Frugal-bribery} and the \textsc{Frugal-\$bribery} problems below thus subsume and strengthen the hardness results of the {\sc \$Bribery} problem from the literature.
 
 \begin{itemize}  
  \item The \textsc{Frugal-bribery} problem is \NPC{} for the Borda voting rule [\Cref{thm:frugalBordaNPC}]. The \textsc{Frugal-\$bribery} is \NPC{} for the $k$-approval for any constant $k \ge 5$ [\Cref{thm:frugalKappNPC}], $k$-veto for any constant $k \ge 3$ [\Cref{thm:frugalKvetoNPC}], and a wide class of scoring rules [\Cref{thm:frugalScrNPC}] even if the price of every vulnerable vote is either $1$ or $\infty$. Moreover, the \textsc{Uniform-frugal-\$bribery} is \NPC{} for the Borda voting rule even if all the vulnerable votes have a uniform
  price of $1$ and the budget is $2$ [\Cref{thm:uniformBordaNPC}].
  
  \item The \textsc{Frugal-\$bribery} problem is \NPshort{}-complete for the Borda, maximin, Copeland, and STV voting rules [\Cref{lem:frugalNPC}].
 \end{itemize}
 
\subsubsection*{Our Results for Weighted Elections}

We have the following results for weighted elections. We observe that, barring a few exceptions, even the most restrictive \textsc{Frugal-bribery} problem is \NPC even when we have only $3$ or $4$ candidates.

\begin{itemize}
 \item The \textsc{Frugal-bribery} problem is in \Pshort{} for the maximin and Copeland voting rules when we have only $3$ candidates [\Cref{lem:wfrugalEasy}], and for the plurality voting rule for any number of candidates [\Cref{thm:wfrugalP}].
 \item The \textsc{Frugal-bribery} problem is \NPC{} for the STV [\Cref{thm:stv_wt}], plurality with runoff [\Cref{cor:run_wt}], and every scoring rule except the plurality voting rule [\Cref{lem:wfrugalScr}] for $3$ candidates. The \textsc{Frugal-\$bribery} problem is \NPC{} for the plurality voting rule for $3$ candidates [\Cref{thm:wfrugalPluNPC}]. 
 \item When we have only $4$ candidates, the \textsc{Frugal-bribery} problem is \NPC{} for the maximin [\Cref{thm:wfrugalMaxmin}], simplified Bucklin [\Cref{thm:bucklin_wt}], and Copeland [\Cref{thm:wfrugalCopeland}] rules.
\end{itemize}

\begin{table}[htbp]\centering
  \resizebox{\linewidth}{!}{
\renewcommand*{\arraystretch}{2}
 \begin{tabular}{|c|c|c|c|c|}\hline 
  \multirow{2}{*}{Voting Rules}			& \multicolumn{2}{c|}{Unweighted}	&\multicolumn{2}{c|}{Weighted}	\\\cline{2-5}
						& \textsc{Frugal-bribery} & \textsc{Frugal-\$bribery}& \textsc{Frugal-bribery} & \textsc{Frugal-\$bribery}\\\hline\hline
  Plurality & \makecell{\Pshort{}\\\relax[\Cref{lem:frugalP}]} & \makecell{\Pshort{}\\\relax[\Cref{lem:frugalPluP}]} & \makecell{\Pshort{}\\\relax[\Cref{thm:wfrugalP}]} & \makecell{\NPC{}\\\relax[\Cref{thm:wfrugalPluNPC}]}\\\hline
  
  Veto & \makecell{\Pshort{}\\\relax[\Cref{lem:frugalP}]} & \makecell{\Pshort{}\\\relax[\Cref{lem:frugalPluP}]} & \makecell{\NPC{}\\\relax[\Cref{lem:wfrugalScr}]} & \makecell{\NPC{}\\\relax[\Cref{lem:wfrugalScr}]}\\\hline
  
  $k$-approval & \makecell{\Pshort{}\\\relax[\Cref{lem:frugalP}]} & \makecell{\NPC{}$^{\star}$\\\relax[\Cref{thm:frugalKappNPC}]} & \makecell{\NPC{}$^{\diamond}$\\\relax[\Cref{lem:wfrugalScr}]} & \makecell{\NPC{}\\\relax[\Cref{lem:wfrugalScr}]}\\\hline
  
  $k$-veto & \makecell{\Pshort{}\\\relax[\Cref{lem:frugalP}]} & \makecell{\NPC{}$^{\star}$\\\relax[\Cref{thm:frugalKvetoNPC}]} & \makecell{\NPC{}$^{\diamond}$\\\relax[\Cref{lem:wfrugalScr}]} & \makecell{\NPC{}\\\relax[\Cref{lem:wfrugalScr}]}\\\hline
  
  Borda & \makecell{\NPC{}\\\relax[\Cref{thm:frugalBordaNPC}]} & \makecell{\NPC{}$^{\dagger}$\\\relax[\Cref{thm:frugalScrNPC}]} & \makecell{\NPC{}\\\relax[\Cref{lem:wfrugalScr}]} & \makecell{\NPC{}\\\relax[\Cref{lem:wfrugalScr}]}\\\hline
  
  Runoff & \makecell{\Pshort{}\\\relax[\Cref{lem:frugalP}]} & ? & \makecell{\NPC{}\\\relax[\Cref{cor:run_wt}]} & \makecell{\NPC{}\\\relax[\Cref{cor:run_wt}]}\\\hline
  
  Maximin & ? & \makecell{\NPC{}\\\relax[\Cref{lem:frugalNPC}]} & \makecell{\NPC{}\\\relax[\Cref{thm:wfrugalMaxmin}]} & \makecell{\NPC{}\\\relax[\Cref{thm:wfrugalMaxmin}]}\\\hline
  
  Copeland & ? & \makecell{\NPC{}\\\relax[\Cref{lem:frugalNPC}]} & \makecell{\NPC{}\\\relax[\Cref{thm:wfrugalCopeland}]} & \makecell{\NPC{}\\\relax[\Cref{thm:wfrugalCopeland}]}\\\hline
  
  STV & ? & \makecell{\NPC{}\\\relax[\Cref{lem:frugalNPC}]} & \makecell{\NPC{}\\\relax[\Cref{thm:stv_wt}]} & \makecell{\NPC{}\\\relax[\Cref{thm:stv_wt}]}\\\hline
\end{tabular}}
  \caption{${\star}$- The result holds for $k\ge3$. $\dagger$- The result holds for a much wider class of scoring rules. ${\diamond}$- The results do not hold for the plurality voting rule. ?- The problem is open.}\label{tbl:frugal_summary}
\end{table}

\subsection{Related Work} The pioneering work of Faliszewski et al.~\cite{faliszewski2006complexity} defined and studied the \textsc{\$Bribery} problem wherein, the input is a set of votes with prices for each vote and the goal is to 
make some distinguished candidate win the election, subject to a budget constraint of the briber. The \textsc{Frugal-\$bribery} problem is the \textsc{\$Bribery} problem with the restriction that the price of every non-vulnerable vote is infinite. Also, the \textsc{Frugal-bribery} problem is a special case of the \textsc{Frugal-\$bribery} problem. Hence, whenever the \textsc{\$Bribery} problem is computationally easy in a setting, both the \textsc{Frugal-bribery} and the \textsc{Frugal-\$bribery} problems are also computationally easy (see \Cref{prop:conn} for a more formal proof). 
However, the \textsc{\$Bribery} problem is computationally intractable in most of the settings. This makes the study of important special cases such as \textsc{Frugal-bribery} and \textsc{Frugal-\$bribery}, interesting. Elkind et al.\cite{ElkindFS09} introduced and studied the {\sc Swap-bribery} problem where we have a more sophisticated cost model specifying cost of swapping every pair of candidates for every vote. It turns out that the \textsc{Frugal-bribery} problem is a special case of the {\sc Swap-bribery} problem (see \Cref{prop:swap} for a formal proof). We note that a notion similar to vulnerable votes has been studied in the context of dominating manipulation by Conitzer et al. \cite{ConitzerWX11}. Hazon et al.~\cite{hazon2013change} introduced and studied \textsc{Persuasion} and $k$-\textsc{Persuasion} problems for the plurality, veto, $k$-approval, Bucklin, and Borda voting rules in unweighted elections only. In the \textsc{Persuasion} and $k$-\textsc{Persuasion} problems an external agent suggests votes to vulnerable voters which are beneficial for the vulnerable voters as well as the external agent. It turns out that the \textsc{Persuasion} and the $k$-\textsc{Persuasion} problems Turing reduce to the \textsc{Frugal-bribery} and the \textsc{Frugal-\$bribery} problems respectively (see \Cref{prop:persu}). Therefore, the polynomial time algorithms we propose in this work imply polynomial time algorithms for the persuasion analog. On the other hand, since the reduction in~\Cref{prop:persu} from \textsc{Persuasion} to \textsc{Frugal-bribery} is a Turing reduction, the existing \NP{}-completeness results for the persuasion problems do not imply \NP{}-completeness results for the corresponding frugal bribery variants. We refer to the book by Rogers~\cite{rogers1967theory} for Turing reductions.

\paragraph*{Organization} The rest of the paper is organized as follows. We first establish the setup and general notions in \Cref{sec:prelim}. Next we present our results for unweighted elections in \Cref{sec:unwt} and our results for the weighted elections \Cref{sec:wt}. Finally we conclude in \Cref{sec:con}. A Preliminary version of this work appeared at AAAI-16~\cite{Deyfrugal}.
\section{Preliminaries}\label{sec:prelim}

Let $\mathcal{V}=\{v_1, \dots, v_n\}$ be the set of all \emph{voters} and $\mathcal{C}=\{c_1, \dots, c_m\}$ 
the set of all \emph{candidates}. 
Each voter $v_i$'s \textit{vote} is a \emph{preference} $\succ_i$ over the 
candidates which is a linear order over $\mathcal{C}$. For example, for two candidates $a$ and $b$, $a \succ_i b$ means that the voter $v_i$ prefers $a$ to $b$. Let $\suc$ be a vote over \CC, $x\in\CC$ be a candidate, and $k$ be a positive integer. We say that $x$ is placed at the $k^{th}$ position in the vote $\suc$ if there are exactly $k-1$ candidates in $\CC\setminus\{x\}$ who are preferred over $x$ in \suc; that is, $|\{y\in\CC\setminus\{x\}:y\suc x\}|=k-1$. We say that $x$ is placed at the top or the first position of $\suc$ if $x$ is preferred over every other candidate $y\in\CC\setminus\{x\}$. We say that $x$ is placed at the bottom or the last position of \suc if every candidate $y\in\CC\setminus\{x\}$ other than $x$ is preferred over $x$ in \suc. We say that $x$ is placed at the $k^{th}$ position from the bottom or the last position if $x$ is preferred over exactly $k-1$ candidates $y\in\CC\setminus\{x\}$ other than $x$; that is, $|\{y\in\CC\setminus\{x\}: x\suc y\}=k-1|$. In this paper, whenever we do not specify the order among a set of candidates while describing a vote, the statement/proof is correct in whichever way we fix the order among them. We denote the set $\{0,1, 2, \ldots\}$ by $\mathbb{N}$, $\mathbb{N}\setminus\{0\}$ by $\mathbb{N}^+$, and $\{1, \ldots, k\}$ by $[k]$, for any positive integer $k$. We denote the set of all linear orders over $\mathcal{C}$ by $\mathcal{L(C)}$. Hence, $\mathcal{L(C)}^n$ denotes the set of all $n$-voters' preference profiles $\succ_{[n]}=(\succ_1, \dots, \succ_n)$. Let $\uplus$ denote the disjoint union of sets. A map $r_c:\uplus_{n,|\mathcal{C}|\in\mathbb{N}^+}\mathcal{L(C)}^n\longrightarrow 2^\mathcal{C}\setminus\{\emptyset\}$ is called a \emph{voting correspondence}. A map $t:2^\mathcal{C}\setminus\{\emptyset\}\longrightarrow \mathcal{C}$ is called a \emph{tie breaking rule}. A commonly used tie breaking rule is the \emph{lexicographic} tie breaking rule where ties are broken according to a predetermined preference $\succ_t \in \mathcal{L(C)}$. A \emph{voting rule} is $r=t\circ r_c$, where $\circ$ denotes the composition of maps. 

\paragraph*{Remark} We note that, in the literature, the definition of a voting rule is usually defined as what we are referring to as a voting correspondence in the above. In particular, the tie-breaking rule is often left out from the definition. We choose to only deal with voting rules that lead to unique winners by definition, because of our notion of vulnerable votes. However, the notion of vulnerable votes can be generalized in natural ways (say, for instance, that a vote is vulnerable if it prefers the desired candidate over all the current winners; or at least one of them). As long as we require $p$ to be the unique winner of the bribed profile, our proofs will carry over to the more general setting. We use tie-breaking rules mostly for ease of presentation. 


In many settings, the voters may have positive integer weights. Such an election is called a weighted election. The winner of a weighted election is defined to be the winner of the unweighted election where each vote is replaced by as many copies of the vote as its weight. We remark that for all the voting rules studied here, the winner of any weighted election can be computed in polynomial amount of time. We assume the elections to be unweighted, if not stated otherwise. Given an election $E$, we can construct a directed weighted graph $G_E$, called the \textit{weighted majority graph}, from $E$. The set of vertices in $G_E$ is the set of candidates in $E$. For any two candidates $x$ and $y$, the weight of the edge $(x,y)$ is $D_E(x,y) = N_E(x,y) - N_E(y,x)$, where $N_E(a, b)$ is the number of voters who prefer candidate $a$ to $b$. A candidate $x$ is called the {\em Condorcet winner} in an election $E$ if $D_E(x,y) > 0$ for every other candidate $y \ne x$. A voting rule is called {\em Condorcet consistent} if it selects the Condorcet winner 
as the winner of the election whenever it exists.
Some examples of common voting correspondences are as follows.

\begin{itemize}
 \item \textbf{Positional scoring rules:} A collection of $m$-dimensional vectors $\overrightarrow{s_m}=\left(\alpha_1,\alpha_2,\dots,\alpha_m\right)\in\mathbb{R}^m$ 
 with $\alpha_1\ge\alpha_2\ge\dots\ge\alpha_m$ and $\alpha_1>\alpha_m$ for every $m\in \mathbb{N}$ naturally defines a
 voting rule --- a candidate gets score $\alpha_i$ from a vote if it is placed at the $i^{th}$ position, and the 
 score of a candidate is the sum of the scores it receives from all the votes. 
 The winners are the candidates with maximum score. Scoring rules remain unchanged if we multiply every $\alpha_i$ by any constant $\lambda>0$ and/or add any constant $\mu$. Hence, we assume without loss of generality that for any score vector $\overrightarrow{s_m}$, there exists a $j$ such that $\alpha_j - \alpha_{j+1}=1$ and $\alpha_k = 0$ for all $k>j$. We call such a $\overrightarrow{s_m}$ a normalized score vector. 
 
 If $\alpha_i$ is $1$ for $i\in [k]$ and $0$ otherwise, then, we get the {\em $k$-approval} voting rule. For the {\em $k$-veto} voting rule, $\alpha_i$ is $0$ for $i\in [m-k]$ and $-1$ otherwise. $1$-approval is called the {\em plurality} voting rule and $1$-veto is called the {\em veto} voting rule. If $\alpha_i=m-i$ for every $i\in[m]$, then we get the {\em Borda} voting rule.
 
 \item \textbf{Maximin:} The maximin score of a candidate $x$ is $\min_{y\ne x} D_E(x,y)$. The winners are the candidates with maximum maximin score.
 
 \item \textbf{Copeland$^{\alpha}$:} Given $\alpha\in[0,1]$, the Copeland$^{\alpha}$ score of a candidate $x$ is $|\{y\ne x:D_E(x,y)>0\}|+\alpha|\{y\ne x:D_E(x,y)=0\}|$. The winners are the candidates with maximum Copeland$^{\alpha}$ score. If not mentioned otherwise, we will assume $\alpha$ to be zero.
 
 \item \textbf{Simplified Bucklin:} A candidate $x$'s simplified Bucklin score is the minimum number $\ell$ such that more than half 
 of the voters rank $x$ in their top $\ell$ positions. The winners are the candidates with lowest simplified Bucklin score.
 
 \item \textbf{Plurality with runoff:} The top two candidates according to the plurality scores are selected first. The pairwise winner of these two candidates is selected as the winner. This rule is often called the runoff voting rule.
 
 \item \textbf{Single Transferable Vote:} In Single Transferable Vote (STV), 
 a candidate with the least plurality score is dropped from the election and its votes 
 are transferred to the next preferred candidate. If two or more candidates receive the least plurality score,  then a tie breaking rule is used. The candidate that remains after $(m-1)$ rounds is the winner.
\end{itemize}

Among the above voting correspondences along with any arbitrary lexicographic tie-breaking rule, only the maximin and the Copeland voting rules are Condorcet consistent.

We use the notation $A\le_\Pshort{}B$ to denote that the problem $A$ polynomial time many-to-one reduces to the problem $B$.

\subsection{Problem Definition}

In all the definitions below, $r$ is a fixed voting rule. We define the notion of vulnerable votes as follows.
Intuitively, the vulnerable votes are those votes whose voters can easily be persuaded to change their votes since doing so will result in an outcome that those voters prefer over the current one.

\begin{definition}(Vulnerable votes)\\
 Given a voting rule $r$, a set of candidates $\CC$, a profile of votes $\succ = (\succ_1, \ldots, \succ_n)$, and a distinguished candidate $p$, we say a vote $\succ_i$ is $p$-vulnerable if $p\succ_i r(\succ)$.
\end{definition} 

Recall that, whenever the distinguished candidate is clear from the context, we drop it from the notation. With the above definition of vulnerable votes, we formally define the \textsc{Frugal-bribery} problem as follows. Intuitively, the problem is to determine whether a particular candidate can be made winner by changing only the vulnerable votes.

\defproblem{$r$-\textsc{Frugal-bribery}}{A set \CC of candidates, a preference profile $\succ = (\succ_1, \ldots, \succ_n)$ over \CC, and a candidate $p$.}{Is there a way to make $p$ win the election according to the voting rule $r$ by changing only the $p$-vulnerable votes?}


We denote an arbitrary instance of $r$-\textsc{Frugal-bribery} by $(\CC, \succ, p)$. Next we generalize the \textsc{Frugal-bribery} problem to the \textsc{Frugal-\$bribery} problem which involves prices for the 
vulnerable votes and a budget for the briber. Intuitively, the price of a vulnerable vote $v$ is the cost the briber incurs to change the vote $v$.

\defproblem{$r$-\textsc{Frugal-\$bribery}}{A set \CC of candidates, a preference profile $\succ = (\succ_1, \ldots, \succ_n)$ over \CC, a candidate $p$, a finite budget $b\in\mathbb{N}$, and a price function $c:[n]\longrightarrow \mathbb{N}\cup\{\infty\}$ such that $c(i)  = \infty$ if $\succ_i$ is not a $p$-vulnerable vote.}{Do there exist $p$-vulnerable votes $\succ_{i_1}, \ldots, \succ_{i_\ell}\in\succ$ and votes $\succ_{i_1}^\prime, \ldots, \succ_{i_\ell}^\prime\in\mathcal{L}(C)$ such that:
\begin{enumerate}
	\item[(a)] the total cost of the chosen votes is within the budget, that is, $\sum_{j=1}^\ell c(i_j) \le b$, and
	\item[(b)] the new votes make the desired candidate win according to the voting rule $r$, that is, $r(\succ_{[n]\setminus\{i_1, \ldots, i_\ell\}}, \succ_{i_1}^\prime, \ldots, \succ_{i_\ell}^\prime) = p$. 
\end{enumerate}}

%
%

The special case of the problem when the prices of all the vulnerable votes are the same is called \textsc{Uniform-frugal-\$bribery}. We refer to the general version as \textsc{Nonuniform-frugal-\$bribery}. If not specified, \textsc{Frugal-\$bribery} refers to the nonuniform version. We denote an arbitrary instance of $r$-\textsc{Frugal-\$bribery} by $(\CC, \succ, p, c(\cdot))$. The above problems are important special cases of the well studied \textsc{\$Bribery} problem. Also, the \textsc{Coalitional-manipulation} problem~\cite{bartholdi1989computational,conitzer2007elections}, one of the classic problems in computational social choice theory, turns out to be a special case of the \textsc{Frugal-\$bribery} problem [see \Cref{prop:conn}]. For the sake of completeness, we include the definitions of these problems here.

\defproblem{$r$-\textsc{\$Bribery}~\cite{faliszewski2009hard}}{A set \CC of candidates, a preference profile $\succ = (\succ_1, \ldots, \succ_n)$ over \CC, a candidate $p$, a price function $c:[n]\longrightarrow \mathbb{N}\cup\{\infty\}$, and a budget $b\in\mathbb{N}$.}{Do there exist votes $\succ_{i_1}, \ldots, \succ_{i_\ell}\in\succ$ and votes $\succ_{i_1}^\prime, \ldots, \succ_{i_\ell}^\prime\in\mathcal{L}(C)$ such that:
\begin{enumerate}
	\item[(a)] the total cost of the chosen votes is within the budget, that is, $\sum_{j=1}^\ell c(i_j) \le b$, and
	\item[(b)] the new votes make the desired candidate win according to the voting rule $r$, that is, $r(\succ_{[n]\setminus\{i_1, \ldots, i_\ell\}}, \succ_{i_1}^\prime, \ldots, \succ_{i_\ell}^\prime) = p$. 
\end{enumerate}}


\defproblem{$r$-\textsc{Coalitional-manipulation}~\cite{bartholdi1989computational,conitzer2007elections}}{A set \CC of candidates, a preference profile $\succ^t = (\succ_1, \ldots, \succ_n)$ of truthful voters over \CC, an integer $\el$ encoded in unary, and a distinguished candidate $p$.}{Does there exist an $\el$ voter preference profile $\succ^\el$ such that the candidate $p$ wins uniquely (does not tie with any other candidate) in the profile $(\succ^t,\succ^\el)$ according to the voting rule $r$?}


The following proposition shows relations among the above problems.
\Cref{prop:conn,prop:man_bri_con,prop:persu} below hold for both weighted and unweighted elections.

\begin{proposition}\label{prop:conn}
 For every voting rule, \textsc{Frugal-bribery} $\le_\Pshort{}$ \textsc{Uniform-frugal-\$bribery} $\le_\Pshort{}$ \textsc{Nonuniform-frugal-\$bribery} 
 $\le_\Pshort{}$ \textsc{\$Bribery}. Also, \textsc{Coalitional-manipulation} $\le_\Pshort{}$ \textsc{Nonuniform-frugal-\$bribery}.
\end{proposition}

\begin{proof}
 In the reductions below, let us assume that the election to start with is a weighted election. Since we do not change the weights of any vote in the reduction and since there is a natural one to one correspondence between the votes of the original instance and the reduced instance, the proof also works for unweighted elections.
 
 Given a \textsc{Frugal-bribery} instance, we construct a \textsc{Uniform-frugal-\$bribery} instance by defining 
 the price of every vulnerable vote to be zero and the budget to be zero. Clearly, the two instances are equivalent. Hence, 
  \textsc{Frugal-bribery} $\le_\Pshort{}$ \textsc{Uniform-frugal-\$bribery}.
  
  \textsc{Uniform-frugal-\$bribery} $\le_\Pshort{}$ \textsc{Nonuniform-frugal-\$bribery} $\le_\Pshort{}$ \textsc{\$Bribery} follows from the 
  fact that  \textsc{Uniform-frugal-\$bribery} is a special case of \textsc{Nonuniform-frugal-\$bribery} which in turn is a 
  special case of \textsc{\$Bribery}.
  
  Given a \textsc{Coalitional-manipulation} instance, we construct a \textsc{Nonuniform-frugal-\$bribery} instance as follows. 
  Let $p$ be the distinguished candidate of the manipulators and $ \succ_f = p\succ others$ be any arbitrary but 
  fixed ordering of the candidates given in the \textsc{Coalitional-manipulation} instance. 
  Without loss of generality, we can assume that $p$ does not win if all the manipulators vote $\succ_f$ (Since this is a 
  polynomially checkable case of \textsc{Coalitional-manipulation}). We define the vote of the manipulators to be $\succ_f$, 
  the distinguished candidate of the campaigner to be $p$, the budget of the campaigner to be zero, 
  the price of the manipulators to be zero (notice that all the manipulators' votes are $p$-vulnerable), and the price of the rest of the vulnerable votes to be one. Clearly, the 
  two instances are equivalent. Hence, \textsc{Coalitional-manipulation} $\le_\Pshort{}$ \textsc{Nonuniform-frugal-\$bribery}.
\end{proof}

Also, the \textsc{Frugal-bribery} problem reduces to the \textsc{Coalitional-manipulation} problem by simply making all vulnerable votes to be manipulators.

\begin{proposition}\label{prop:man_bri_con}
 For every voting rule, \textsc{Frugal-bribery} $\le_\Pshort{}$ \textsc{Coalitional-manipulation}.
\end{proposition}

The \textsc{Frugal-bribery} problem also reduces to the \textsc{Swap-bribery} problem as proved below.

\begin{proposition}\label{prop:swap}
 For every voting rule, \textsc{Frugal-bribery} $\le_\Pshort{}$ \textsc{Swap-bribery}.
\end{proposition}

\begin{proof}
 Given an arbitrary instance of the {\sc Frugal-bribery} problem, we define the {\sc Swap-bribery} instance simply by defining the cost of every swap in vulnerable votes to be $0$, the cost of every swap in non-vulnerable votes to be $1$, and the budget to be $0$.
\end{proof}

We can also establish the following relation between the \textsc{Persuasion} (respectively $k$-\textsc{Persuasion}) problem and the \textsc{Frugal-bribery} (respectively \textsc{Frugal-\$bribery}) problem. The persuasions differ from the corresponding frugal bribery variants in that the briber has her own preference order, and desires to improve the outcome of the election with respect to her preference order. The following proposition is immediate from the definitions of the problems. 
\begin{proposition}\label{prop:persu}
 For every voting rule, there is a Turing reduction from \textsc{Persuasion} (respectively \textsc{$k$-persuasion}) to \textsc{Frugal-bribery} (respectively \textsc{Frugal-\$bribery}).
\end{proposition}

\begin{proof}
 Given an algorithm for the \textsc{Frugal-bribery} problem, we iterate over all possible distinguished candidates to have an algorithm for the persuasion problem. 
 
 Given an algorithm for the \textsc{Frugal-\$bribery} problem, we iterate over all possible distinguished candidates and fix the price of the corresponding vulnerables to be one to have an algorithm for the $k$-persuasion problem.
\end{proof}
\section{Results for Unweighted Elections}\label{sec:unwt}

\aaaiversion{In this supplementary material, we often break some results in the shorter version into separable parts for improve presentation.}
Now we present the results for unweighted elections. We begin with some easy observations that follow from known results. 

\shortversion{
 The following result follows immediately from the literature on the \textsc{Coalitional-manipulation}~\cite{xia2009complexity}, the \textsc{\$Bribery} problems~\cite{faliszewski2006complexity,faliszewski2008nonuniform} and~\Cref{prop:conn,prop:man_bri_con}.
\begin{observation}\shortversion{[$\star$]}\label{lem:frugal}
 \begin{itemize}[noitemsep,leftmargin=*]
  \item The \textsc{Frugal-bribery} problem is in \Pshort{} for the $k$-approval for any $k$, simplified Bucklin, and plurality with runoff voting rules.
  \item The \textsc{Frugal-\$bribery} problem is in \Pshort{} for the plurality and veto voting rules.
  \item The \textsc{Frugal-\$bribery} problem is \NPC{} for the Borda, maximin, Copeland, and STV voting rules.
 \end{itemize}
\end{observation}
}\aaaiversion{We divide Observation $1$ of short version into Observation $1, 2,$ and $3$ below.}
\longversion{
\begin{observation}\label{lem:frugalP}
 The \textsc{Frugal-bribery} problem is in \Pshort{} for the $k$-approval voting rule for any $k$, simplified Bucklin, and plurality with runoff voting rules.
\end{observation}

\begin{proof}
 The \textsc{Coalitional-manipulation} problem is in \Pshort{} for these voting rules~\cite{xia2009complexity}. Hence, the result follows from~\Cref{prop:man_bri_con}.
\end{proof}

\begin{observation}\label{lem:frugalPluP}
 The \textsc{Frugal-\$bribery} problem is in \Pshort{} for the plurality and veto voting rules.
\end{observation}

\begin{proof}
 The \textsc{\$Bribery} problem is in \Pshort{} for the plurality~\cite{faliszewski2006complexity} and 
 veto~\cite{faliszewski2008nonuniform} voting rules. Hence, the result follows from~\Cref{prop:conn}.
\end{proof}

\begin{observation}\label{lem:frugalNPC}
 The \textsc{Frugal-\$bribery} problem is \NPC{} for Borda, maximin, Copeland, and STV voting rules.
\end{observation}

\begin{proof}
 The \textsc{Coalitional-manipulation} problem is \NPC{} for the above voting rules. Hence, the result follows from~\Cref{prop:conn}.
\end{proof}
}
 We now present our main results. We begin with showing that the \textsc{Frugal-bribery} problem for the Borda voting rule. To this end, we reduce from the \PS problem, which is known to be \NPC{}~\cite{Yu:2004:MMT:1013651.1013680}. The \PS problem is defined as follows.

\defproblem{\PS}{$n$ integers $X_i, i\in[n]$ with $1\le X_i\le 2n$ for every $i\in[n]$ and $\sum_{i=1}^n X_i = n(n+1)$.}{Do there exist two permutations $\pi$ and $\sigma$ of $[n]$ such that $\pi(i)+\sigma(i)=X_i$ for every $i\in[n]?$}


We now prove that the \textsc{Frugal-bribery} problem is \NPC{} for the Borda voting rule, by a reduction from \PS. Our reduction is inspired by the reduction used by Davies et al.~\cite{davies2011complexity} and Betzler et al.~\cite{betzler2011unweighted} to prove \NP-completeness of the \textsc{Coalitional-manipulation} problem for the Borda voting rule.

\begin{theorem}\label{thm:frugalBordaNPC}
 The \textsc{Frugal-bribery} problem is \NPC{} for the Borda voting rule.
\end{theorem}

\begin{proof}
 The problem is clearly in \NPshort{}. To show \NPshort{}-hardness, we reduce an arbitrary instance of the \PS problem to the \textsc{Frugal-bribery} problem for the Borda voting rule. Let $(X_1, \ldots, X_n)$ be an instance of the \PS problem. Without loss of generality, let us assume that $n$ is an odd integer -- if $n$ is an even integer, then we consider the instance $(X_1, \ldots, X_n, X_{n+1}=2(n+1))$ which is clearly equivalent to the instance $(X_1, \ldots, X_n).$
 
 We define a \textsc{Frugal-bribery} instance $(\CC, \PP, p)$ as follows. The candidate set is: 
 
 $$\mathcal{C} = \XX \uplus D\uplus \{p,c\}, \text{ where } \XX =  \{\xxx_i: i\in[n]\} \text{ and }|D| = 3n - 1$$ 
 
 Note that the total number of candidates is $4n+1$, and therefore the Borda score of a candidate placed at the top position is $4n$.
 
Before describing the votes, we give an informal overview of how the reduction will proceed. The election that we define will consist of exactly two vulnerable votes. Note that when placed at the top position in these two votes, the distinguished candidate $p$ gets a score of $8n$ ($4n$ from each vulnerable vote). We will then add non-vulnerable votes, which will be designed to ensure that, among them, the score of $\xxx_i$ is $8n-X_i$ more than the score of the candidate $p$. Using the ``dummy candidates'', we will also be able to ensure that the candidates $\xxx_i$ receive (without loss of generality) scores between $1$ and $n$ from the modified vulnerable votes. 

Now suppose these two vulnerable votes can be modified to make $p$ win the election. Let $s_1$ and $s_2$ be the scores that $\xxx_i$ obtains from these altered vulnerable votes. It is clear that for $p$ to emerge as a winner, $s_1+s_2$ must be at most $X_i$. Since the Borda scores for the candidates in \XX range from $1$ to $n$ in the altered vulnerable votes, the total Borda score that all the candidates in \XX can accumulate from two altered vulnerable votes is $n(n+1)$. On the other hand, since the sum of the $X_i$'s is also $n(n+1)$, it turns out that $s_1+s_2$ must in fact be equal to $X_i$ for the candidate $p$ to win. From this point, it is straightforward to see how the permutations $\sigma$ and $\pi$ can be inferred from the modified vulnerable votes: $\sigma(i)$ is given by the score of the candidate $\xxx_i$ from the first vote, while $\pi(i)$ is the score of the candidate $\xxx_i$ from the second vote. These functions turn out to be permutations because these $n$ candidates receive $n$ distinct scores from these votes. 
 
We are now ready to describe the construction formally. We remark that instead of $8n-X_i$, as described above, we will maintain a score difference of either $8n-X_i$ or $8n-X_i-1$ depending on whether $X_i$ is even or odd respectively --- this is a minor technicality that comes from the manner in which the votes are constructed and does not affect the overall spirit of the reduction. 

Let us fix any arbitrary order $\succ_f$ among the candidates in $\XX \uplus D.$ For any subset $A\subset \XX \uplus D,$ let $\overrightarrow{A}$ be the ordering among the candidates in $A$ as defined in $\succ_f$ and $\overleftarrow{A}$ the reverse order of $\overrightarrow{A}$. For each $i\in[n]$, we add two votes $v_i^j$ and $v_i^{j^\pr}$ as follows for every $j\in[4]$. Let $\el$ denote $|D| = 3n-1$. Also, for $d\in D$, let $D_{i}, D_{\nfrac{\el}{2}}\subset D\setminus\{d\}$ be such that: $$|D_{i}| = \nfrac{\el}{2}+n+1-\lceil\nfrac{X_i}{2}\rceil \mbox{ and } |D_{\nfrac{\el}{2}}| = \nfrac{\el}{2}.$$

 \[ v_i^j : 
 \begin{cases}
 c \suc p \suc d \suc \overrightarrow{\CC\setminus(\{d, c, p, \xxx_i\}\uplus D_{i})} \suc \xxx_i \suc  \overrightarrow{D_{i}} & \text{for } 1\le j\le 2\\
 \xxx_i \suc \overleftarrow{D_{i}} \suc \overleftarrow{\CC\setminus(\{d, c, p, \xxx_i\}\uplus D_{i})} \suc c \suc p \suc d & \text{for } 3\le j\le 4
 \end{cases}
 \]
 
 \[ v_i^{j^\pr} : 
 \begin{cases}
 c \suc p \suc d \suc \overrightarrow{\CC\setminus(\{d, c, p, \xxx_i\}\uplus D_{\nfrac{\el}{2}})} \suc \xxx_i \suc  \overrightarrow{D_{\nfrac{\el}{2}}} & \text{for } 1\le j^\pr \le 2\\
 \xxx_i \suc \overleftarrow{D_{\nfrac{\el}{2}}} \suc \overleftarrow{\CC\setminus(\{d, c, p, \xxx_i\}\uplus D_{\nfrac{\el}{2}})} \suc c \suc p \suc d & \text{for } 3\le j^\pr \le 4
 \end{cases}
 \]

%
 
It is convenient to view the votes corresponding to $j = 3,4$ as a near-reversal of the votes in $j = 1,2$ (except for candidates $c,d$ and $\xxx_i$).  Let $\PP_1 = \{v_i^j, v_i^{j^\pr} : i\in[n], j\in[4]\}.$ Since there are $8n$ votes in all, and $c$ always appears immediately before $p$, it follows that the score of $c$ is exactly $8n$ more than the score of the candidate $p$ in $\PP_1$. 

We also observe that the score of the candidate $\xxx_i$ is exactly $2(\el+n+1)-X_i = 8n-X_i$ more than the score of the candidate $p$ in $\PP_1$ for every $i\in[n]$ such that $X_i$ is an even integer. On the other hand, the score of the candidate $\xxx_i$ is exactly $2(\el+n+1)-X_i-1 = 8n-X_i-1$ more than the score of the candidate $p$ in $\PP_1$ for every $i\in[n]$ such that $X_i$ is an odd integer. Note that for $i^\pr \in [n] \setminus \{i\}$, $p$ and $\xxx_i$ receive the same Borda score from the votes $v_{i^\pr}^j$ and $v_{i^\pr}^{j^\pr}$ (where $j,j^\pr \in [4]$).

We now add the following two votes $\mu_1$ and $\mu_2$. 
 
\[ \mu_1 : p \succ c \succ \text{others} \]
\[ \mu_2 : p \succ c \succ \text{others} \]

 Let $\PP = \PP_1 \uplus \{\mu_1, \mu_2\}, \XX^o = \{\xxx_i : i\in[n], X_i \text{ is odd}\},$ and $\XX^e = \XX\setminus\XX^o.$ We recall that the distinguished candidate is $p.$ The tie-breaking rule is according to the order $\XX^o \suc p \succ \text{others}.$ We claim that the \textsc{Frugal-bribery} instance $(\CC, \PP, p)$ is equivalent to the \PS instance $(X_1, \ldots, X_n).$
 
 In the forward direction, suppose there exist two permutations $\pi$ and $\sigma$ of $[n]$ such that $\pi(i) + \sigma(i) = X_i$ for every $i\in[n].$ We replace the votes $\mu_1$ and $\mu_2$ with respectively $\mu_1^\pr$ and $\mu_2^\pr$ as follows.
 
 \[ \mu_1^\pr : p \suc D \suc \xxx_{\pi^{-1}(n)} \suc \xxx_{\pi^{-1}(n-1)} \suc \cdots \suc \xxx_{\pi^{-1}(1)} \suc c \]
 \[ \mu_2^\pr : p \suc D \suc \xxx_{\sigma^{-1}(n)} \suc \xxx_{\sigma^{-1}(n-1)} \suc \cdots \suc \xxx_{\sigma^{-1}(1)} \suc c \]
 
 We observe that, the candidates $c$ and every $\xxx\in\XX^e$ receive same score as $p$, every candidate $\xxx^\pr\in\XX^o$ receives $1$ score less than $p$, and every candidate in $D$ receives less score than $p$ in $\PP_1\uplus\{\mu_1^\pr, \mu_2^\pr\}.$ Hence $p$ wins in $\PP_1\uplus\{\mu_1^\pr, \mu_2^\pr\}$ due to the tie-breaking rule. Thus $(\CC, \PP, p)$ is a \YES instance of \textsc{Frugal-bribery}.
 
 To prove the other direction, suppose the \textsc{Frugal-bribery} instance is a \YES{} instance. Notice that the only vulnerable votes are $\mu_1$ and $\mu_2.$ Let $\mu_1^\pr$ and $\mu_2^\pr$ be two votes such that the candidate $p$ wins in the profile $\PP_1\uplus\{\mu_1^\pr, \mu_2^\pr\}.$ We assume, without loss of generality, that candidate $p$ is placed at the first position in both $\mu_1^\pr$ and $\mu_2^\pr.$ Since $c$ receives $8n$ scores more than $p$ in $\PP_1,$ $c$ must be placed at the last position in both $\mu_1^\pr$ and $\mu_2^\pr$ since otherwise $p$ cannot win in $\PP_1\uplus\{\mu_1^\pr, \mu_2^\pr\}.$ We also assume, without loss of generality, that every candidate in $D$ is preferred over every candidate in \XX since otherwise, if $\xxx \suc d$ in either $\mu_1^\pr$ or $\mu_2^\pr$ for some $\xxx\in\XX$ and $d\in D,$ then we can exchange the positions of \xxx and $d$ and $p$ continues to win since no candidate in $D$ receives more score than $p$ in $\PP_1.$ Hence, every $\xxx\in\XX$ receives some score between $1$ and $n$ in both the $\mu_1^\pr$ and $\mu_2^\pr.$ Let us define two permutations $\pi$ and $\sigma$ of $[n]$ as follows. For every $i\in[n]$, we define $\pi(i)$ and $\sigma(i)$ to be the scores the candidate $\xxx_i$ receives in $\mu_1^\pr$ and $\mu_2^\pr$ respectively. The fact that $\pi$ and $\sigma$, as defined above, is indeed a permutation of $[n]$ follows from the structure of the votes $\mu_1^\pr, \mu_2^\pr$ and the Borda score vector. Since $p$ wins in $\PP_1\uplus\{\mu_1^\pr, \mu_2^\pr\},$ we have $\pi(i) + \sigma(i) \le X_i.$ We now have the following.
 
 \[ n(n+1) = \sum_{i=1}^n (\pi(i) + \sigma(i)) \le \sum_{i=1}^n X_i = n(n+1) \]
 
 Hence, we have $\pi(i) + \sigma(i) = X_i$ for every $i\in[n]$ and thus $(X_1, \ldots, X_n)$ is a \YES instance of \PS.
\end{proof}

We will use \Cref{score_gen} in subsequent proofs, which has been shown before (see, for instance, the work of Baumeister et al.~\cite{baumeister2011computational} and Dey et al.~\cite{Deykernel}).
\begin{lemma}\shortversion{[$\star$]}\label{score_gen}
Let $\mathcal{C} = \{c_1, \ldots, c_m\} \uplus D, (|D|>0)$ be a set of candidates and $\overrightarrow{\alpha}$ a normalized score vector of length $|\mathcal{C}|$. Then, for any given $\mathbf{X} = (X_1, \ldots, X_m) \in \mathbb{Z}^m$, there exists $\lambda\in \mathbb{R}$ and a voting profile such that the $\overrightarrow{\alpha}$-score of $c_i$ is $\lambda + X_i$ for all $1\le i\le m$,  and the score of candidates $d\in D$ is less than $\lambda$. Moreover, the number of votes is $O(poly(|\mathcal{C}|\cdot \sum_{i=1}^m |X_i|))$, where $|X_i|$ is the absolute value of $X_i$.
\end{lemma}

 Note that the number of votes used in \Cref{score_gen} is polynomial in $m$ if $|D|$ and $|X_i|$ are polynomials in $m$ for every $i\in [m]$, which indeed is the case in all our proofs that use \Cref{score_gen}. Hence, the reductions in the proofs that use \Cref{score_gen} run in polynomial time.
 
 We now show the results for various classes of scoring rules. To this end, we reduce from the \textsc{Exact-cover-by-3-sets} (X3C) problem, which is known to be \NPC{} \cite{garey1979computers}. The X3C problem is defined as follows.
 
\defproblem{X3C}{A universe $U$ and $t$ subsets $S_1, \dots, S_t \subset U$ with $|S_i|=3 ~\forall i\in[t].$}{Does there exist an index set $I\subseteq [t]$ 
 with $|I|=\nfrac{|U|}{3}$ such that $\uplus_{i\in I} S_i = U$?}

We denote an arbitrary instance of X3C by $(U, \{S_1, \dots, S_t\})$.

\begin{theorem}\label{thm:frugalKappNPC}
 The \textsc{Frugal-\$bribery} problem is \NPC{} for the $k$-approval voting rule for any constant $k\ge 3$, even if the price of every vulnerable vote is either $1$ or $\infty$.
\end{theorem}

\begin{proof}
 The problem is clearly in \NPshort{}. To show \NPshort{}-hardness, we reduce an arbitrary instance of X3C to \textsc{Frugal-\$bribery}. 
 Let $(U, \{S_1, \dots, S_t\})$ be an instance of X3C. We define a \textsc{Frugal-\$bribery} instance as follows. 
 The candidate set is: 
 
 \[\mathcal{C} = U\uplus D\uplus \{p,q\},\text{ where }|D|=k-1\]
 
 For each $S_i, 1\le i\le t$, we add a vote $v_i$ as follows.
 
 \[ v_i : S_i \succ D \succ p \suc q \suc \text{others} \]
 
 By \Cref{score_gen}, we can add $poly(|U|)$ many additional votes to ensure the following scores (denoted by $s(\cdot)$).
 
 \begin{itemize}
 	\item $s(q) = s(p) + \nfrac{|U|}{3}$
 	\item $s(x) = s(p) + \nfrac{|U|}{3} + 1, \forall x\in U$
 	\item $s(d) < s(p) - |U|, \forall d\in D$
 \end{itemize}
 
 The tie-breaking rule is ``$p \succ \text{others}$''. The winner is $q$. The distinguished candidate is $p$ and thus all the votes in $\{v_i : 1\le i\le t\}$ are vulnerable. The price of every $v_i$ is $1$ and the price of every other vulnerable vote is $\infty$. The budget is $\nfrac{|U|}{3}$. This completes the construction. We now prove that the two instances are equivalent. 
 
 In the forward direction, let us suppose that there exists an index set $I\subseteq [t]$ with $|I|=\nfrac{|U|}{3}$ such that $\uplus_{i\in I} S_i = U$. We replace the votes $v_i$ with $v_i^{\prime}, i\in I,$ which are defined as follows.
 \[ v_i^{\prime} : \underbrace{p \succ D}_{k \text{ candidates}} \succ \text{others} \]
 This makes the score of $p$ not less than the score of any other candidate and thus $p$ wins. 
 
\longversion{To prove the result in}\shortversion{For} the other direction, let us suppose that the \textsc{Frugal-\$bribery} instance is a \YES{} instance. Then there exists $\VV\subset\{v_i : 1\le i\le t\}$ with $|\VV|=\nfrac{|U|}{3}$ such that no vote in $\{v_i : 1\le i\le t\}\setminus\VV$ has been changed by the briber. Let the vote that replaces $v\in\VV$ be $v^\pr$ and let $\VV^\pr = \{v^\pr: v\in\VV\}.$ We assume, without loss of generality, that the candidate $p$ is placed within the first $k$ positions of every vote $v^\pr\in\VV^\pr$. Hence, the final score of the candidate $p$ is $s(p)+\nfrac{|U|}{3}$. We observe that, in every vote $v_i^\pr\in\VV^\pr,$ the candidate $q$ and the corresponding $S_i$ should not be placed within the top $k$ positions since $s(p)+\nfrac{|U|}{3} = s(q)$ and $s(p)+\nfrac{|U|}{3} = s(x)-1$ for every $x\in U.$ We claim that the $S_i$'s corresponding to the $v_i$'s in \VV form an exact set cover. Indeed, otherwise, there will be a candidate $x\in U$, whose score never decreases which contradicts the fact that $p$ wins the election since $s(p)+\nfrac{|U|}{3} = s(x)-1$.
\end{proof}

We next present a similar result for the $k$-veto voting rule\shortversion{ which can be proved by a reduction from the X3C problem}.
 
\begin{theorem}\shortversion{[$\star$]}\label{thm:frugalKvetoNPC}
 The \textsc{Frugal-\$bribery} problem is \NPC{} for the $k$-veto voting rule for any constant $k\ge3$, even if the price of every vulnerable vote is either $1$ or $\infty$.
\end{theorem}

\longversion{ \begin{proof}
 The problem is clearly in \NPshort{}. To show \NPshort{}-hardness, we reduce an arbitrary instance of X3C to \textsc{Frugal-\$bribery}. Let $(U, \{S_1,S_2, \dots, S_t\})$ be any instance of X3C. We define a \textsc{Frugal-\$bribery} instance as follows. The candidate set is: 
 
 \[\mathcal{C} = U\uplus Q\uplus \{p, a_1, a_2, a_3, d\},\text{ where }|Q|=k-3\]
 
 For each $S_i, 1\le i\le t$, we add a vote $v_i$ as follows.
 \[ v_i : p \succ \text{others} \succ \underbrace{S_i \succ Q}_{k \text{ candidates}} \]
 By \Cref{score_gen}, we can add $poly(|U|)$ many additional votes to ensure following scores (denoted by $s(\cdot)$).
 
 \begin{itemize}
 	\item $s(p) > s(d), s(p) = s(x) + 2, \forall x\in U$
 	\item $s(p) = s(q) + 1, \forall q\in Q$
 	\item $s(p) = s(a_i) - \nfrac{|U|}{3} + 1, \forall i=1, 2, 3$
 \end{itemize}

%
 
 The tie-breaking rule is ``$a_1 \succ \cdots \succ p$''. The winner is $a_1$.
 The distinguished candidate is $p$ and thus all the votes in $\{v_i : 1\le i\le t\}$ are vulnerable. 
 The price of every $v_i$ is one and the price of any other vote is $\infty$. The budget is $\nfrac{|U|}{3}$. We claim that the two instances are equivalent. 
 
 In the forward direction, suppose there exists an index set $I\subseteq \{1,\dots,t\}$ with $|I|=\nfrac{|U|}{3}$ such that $\uplus_{i\in I} S_i = U$. We replace the votes $v_i$ with $v_i^{\prime}, i\in I,$ which are defined as follows. 
 
 \[ v_i^{\prime} : \text{others} \succ \underbrace{a_1 \succ a_2 \succ a_3 \succ Q}_{k \text{ candidates}} \]
 
 The score of each $a_i$ decreases by $\nfrac{|U|}{3}$ and their final scores are $s(p)-1$, since the score of $p$ is not affected 
 by this change. Also the final score of each $x\in U$ is $s(p)-1$ since $I$ forms an exact set cover. This makes $p$ win the election.
 
To prove the result in the other direction, suppose the \textsc{Frugal-\$bribery} instance is a \YES{} instance. Then, notice that there will be exactly $\nfrac{|U|}{3}$ votes in $v_i, 1\le i\le t$, where every $a_j, j=1, 2, 3$, should come in the last $k$ positions since $s(p) = s(a_j) - \nfrac{|U|}{3} + 1$ and the budget is $\nfrac{|U|}{3}$. Notice that candidates in $Q$ must not be placed within top $m-k$ positions since $s(p) = s(q) + 1$, for every $q\in Q$. Hence, in the votes that have been changed, $a_1, a_2, a_3$ and all the candidates in $Q$ must occupy the last $k$ positions. We claim that the $S_i$'s corresponding to the $v_i$'s that have been changed must form an exact set cover. If not, then, there must exist a candidate $x\in U$ and two votes $v_i$ and $v_j$ such that, both $v_i$ and $v_j$ have been replaced by $v_i^{\prime} \ne v_i$ and $v_j^{\prime} \ne v_j$ and the candidate $x$ was present within the last $k$ positions in both $v_i$ and $v_j$. This makes the score of $x$ at least the score of $p$ which contradicts the fact that $p$ wins.
\end{proof}}

 We now show that there exists a polynomial time algorithm for the \textsc{Frugal-\$bribery} problem for the $k$-approval, simplified Bucklin, and plurality with runoff voting rules, when the budget is a constant. The result below follows from the existence of a polynomial time algorithm for the \textsc{Coalitional-manipulation} problem for these voting rules for any number of manipulators~\cite{xia2009complexity}.
 
\begin{theorem}\shortversion{[$\star$]}\label{thm:frugalKappP}
 The \textsc{Frugal-\$bribery} problem is in \Pshort{} for the $k$-approval, simplified Bucklin, and plurality with runoff voting rules, if the budget is a constant.
\end{theorem}

\longversion{\begin{proof}
 Let the budget $b$ be a constant. Then, at most $b$ many vulnerable votes whose price is not zero can be changed since the prices are assumed to be in $\mathbb{N}$. Notice that we may assume, 
 without loss of generality, that all the vulnerable votes whose price is zero will be changed. 
 We iterate over all the $O(n^b)$ many possible vulnerable vote changes and we can solve each one 
 in polynomial time since the \textsc{Coalitional-manipulation} problem is in \Pshort{} for these voting rules~\cite{xia2009complexity}.
\end{proof}}

\longversion{We show that the \textsc{Frugal-\$bribery} problem is \NPC{} for a wide class of scoring rules as characterized in the following
result.}\shortversion{ Our next result shows that, the \textsc{Frugal-\$bribery} problem is \NPC{} for a wide class of scoring rules\longversion{ that includes the Borda voting rule}. \Cref{thm:frugalScrNPC} can be proved by a reduction from the X3C problem.}

\begin{theorem}\shortversion{[$\star$]}\label{thm:frugalScrNPC}
 For any positional scoring rule $r$ with score vectors $\{\overrightarrow{s_i} : i\in \mathbb{N}\}$, if there exists a polynomial function $f: \mathbb{N} \longrightarrow \mathbb{N}$ such that, for every $m\in \mathbb{N}, f(m) \ge 2m$ and in the score vector $(\alpha_1, \ldots, \alpha_{f(m)})$, there exists a $m\le \ell\le f(m)-5$ satisfying the following condition:
 \[ \alpha_i - \alpha_{i+1} = \alpha_{i+1} - \alpha_{i+2} > 0, \forall \ell \le i \le \ell+3 \]
 then the \textsc{Frugal-\$bribery} problem is \NPC{} for $r$ even if the price of every vulnerable vote is either $1$ or $\infty$.
\end{theorem}

\longversion{
\begin{proof}
 The problem is clearly in \NPshort{}. To show \NPshort{}-hardness, we reduce an arbitrary instance of X3C to \textsc{Frugal-\$bribery}. 
 Let $(U, \{S_1, \dots, S_t\})$ be an instance of X3C. We define a \textsc{Frugal-\$bribery} instance as follows. Let us consider the score vector $(\alpha_1, \ldots, \alpha_{f(|U|)}).$ Since the scoring rules remain unchanged if we multiply every $\alpha_i$ by any constant $\lambda>0$ and/or add any constant $\mu$, 
 we can assume the following without loss of generality.
 
 \[ \alpha_i - \alpha_{i+1} = \alpha_{i+1} - \alpha_{i+2} = 1, \forall \ell \le i \le \ell+3 \]
 
 The candidate set is:
 
 \[\mathcal{C} = U \uplus Q \uplus \{p, a, d\},\text{ where } |Q| = f(|U|)-|U|-4 \text{ and } Q = \{q_1, q_2, \ldots, q_{|Q|}\}\]
 
 Let us fix any arbitrary order $\succ_f$ among the candidates in $U \uplus Q.$ For any subset $A\subset U \uplus Q,$ let $\overrightarrow{A}$ be the ordering among the candidates in $A$ as defined in $\succ_f.$ For each $S_i = \{x,y,z\}, 1\le i\le t$, we add a vote $v_i$ as follows.
 
 \[ v_i : p \succ d \succ \overrightarrow{\text{others}} \succ \underbrace{a \succ x\succ y\succ z \succ q_1 \suc q_2 \suc \cdots \suc q_{f(|U|)-\el-4}}_{l \text{ candidates}} \]
 
 By \Cref{score_gen}, we can add $poly(|U|)$ many additional votes to ensure the following scores (denoted by $s(\cdot)$) in the resulting profile (including the votes $v_i, i\in[t]$). Note that the proof of \Cref{score_gen} by Baumeister et al.~\cite{baumeister2011computational} also works for the normalization of $\alpha$ defined in the beginning of the proof.
 
 \begin{itemize}
 	\item $s(d) < s(p)$
 	\item $s(x) = s(p) - 2, \forall x\in U$
 	\item $s(a) = s(p) + \nfrac{|U|}{3} - 1$
 	\item $s(q) = s(p) - 1, \forall q\in Q$
 \end{itemize}
 
 
 The tie-breaking rule is ``$\cdots \succ p$.'' The candidate $a$ wins. The distinguished candidate is $p$. The price of every $v_i$ is $1$ and the price of every other vulnerable vote is $\infty$. The budget is $\nfrac{|U|}{3}$. We claim that the two instances are equivalent.
 
 In the forward direction, there exists an index set $I\subseteq [t], |I|=\nfrac{|U|}{3},$ such that $\uplus_{i\in I} S_i = U$. We replace the votes $v_i$ with $v_i^{\prime}, i\in I,$ which are defined as follows.
 
 \[ v_i^{\prime} : p \succ d \succ \overrightarrow{\text{others}} \succ x\succ y\succ z \succ a \succ q_1 \suc q_2 \suc \cdots \suc q_{f(|U|)-\el-4} \]
 
 This makes the score of $p$ at least one more than the score of every other candidate and thus $p$ wins. 
 
 To prove the result in the other direction, let us suppose that the \textsc{Frugal-\$bribery} instance is a \YES{} instance. Then there exists $\VV\subset\{v_i : 1\le i\le t\}$ with $|\VV|=\nfrac{|U|}{3}$ such that no vote in $\{v_i : 1\le i\le t\}\setminus\VV$ has been changed by the briber. Let the vote that replaces $v\in\VV$ be $v^\pr$ and let $\VV^\pr = \{v^\pr: v\in\VV\}.$ Let the resulting profile be $\PP^\pr.$ We first claim that the candidate $q_{f(|U|)-\el-4}$ is placed at the last position of every $v^\pr\in\VV^\pr.$ Indeed, otherwise the score of the candidate $q_{\el-4}$ is not less than the score of $p$ in $\PP^\pr$ which contradicts our assumption that the candidate $p$ wins in $\PP^\pr$ since the tie-breaking rule is ``$\cdots \succ p$.'' Given $q_{f(|U|)-\el-4}$ is placed at the last position of every $v^\pr\in\VV^\pr,$ we observe, by the same argument applied for the candidate $q_{f(|U|)-\el-5}$, that the candidate $q_{f(|U|)-\el-5}$ is placed in the second last position of every $v^\pr\in\VV^\pr.$ Arguing similarly all the way to the candidate $q_1$ we observe that the last $(f(|U|)-\el-4)$ positions of every $v^\pr\in\VV^\pr$ will be $q_1 \suc q_2 \suc \cdots \suc q_{f(|U|)-\el-4}.$ Since $s(p) = s(a) - \nfrac{|U|}{3} + 1$ and the tie-breaking rule is ``$\cdots \succ p,$'' the candidate $a$ must be placed at the $(l+4)^{th}$ position in every $v^\pr\in\VV^\pr.$ Hence, for every $i\in[t]$ such that $v_i^\pr\in\VV^\pr,$ if $S_i=\{x,y,z\},$ then the scores of the candidates $x, y,$ and $z$ increase by at least $1$ each. We claim that $\cup_{i:v_i^\pr\in\VV^\pr} S_i = U.$ If not, then there must exist a candidate $x\in U$ whose score has increased by at least $2$ contradicting the fact that $p$ wins in $\PP^\pr.$
\end{proof}
}

 For the sake of concreteness, an example of a function $f$, stated in \Cref{thm:frugalScrNPC}, that works for the Borda voting rule is $f(m)=2m$. \longversion{\Cref{thm:frugalScrNPC} shows that the \textsc{Frugal-\$bribery} problem is intractable for the Borda voting rule. However, the following theorem shows the intractability of the \textsc{Uniform-frugal-\$bribery} problem for the Borda voting rule, even in a very restricted setting.}\shortversion{\Cref{thm:uniformBordaNPC} below shows the intractability of the \textsc{Uniform-frugal-\$bribery} problem for the Borda voting rule, even in a very restricted setting.} 
\Cref{thm:uniformBordaNPC} below is proved by a reduction from the \textsc{Coalition manipulation} problem for the Borda voting rule for two manipulators which is known to be \NP{}-complete~\cite{betzler2011unweighted,davies2011complexity}.

\begin{theorem}\shortversion{[$\star$]}\label{thm:uniformBordaNPC}
 The \textsc{Uniform-frugal-\$bribery} problem is \NPC{} for the Borda voting rule, even when every vulnerable vote has a price of $1$ and the budget is $2$.
\end{theorem}

\longversion{ \begin{proof}
 The problem is clearly in \NPshort{}. To show \NPshort{}-hardness, we reduce an arbitrary instance of the \textsc{Coalitional-manipulation} problem for the Borda voting rule with two manipulators to an instance of the \textsc{Uniform-frugal-\$bribery} problem for the Borda voting rule. Let $(C, \succ^t, 2, p)$ be an arbitrary instance of the \textsc{Coalitional-manipulation} problem for the Borda voting rule and $|C|=m.$ The corresponding \textsc{Frugal-\$bribery} instance is as follows. The candidate set is:
 
 \[C^{\prime} = C \uplus \{d,q\}\]
 
 For each vote $v_i \in \succ^t$, we add a vote $v_i^{\prime}$ as follows.
 
 \[v_i^{\prime} : v_i \succ d \succ q\]
 
 Let $\overrightarrow{C\setminus \{p\}}$ is an arbitrary but fixed order of the candidates in $C\setminus \{p\}$. Corresponding to the two manipulators', we add two more votes $\nu_1$ and $\nu_2$ as follows. 
 
 \[\nu_1, \nu_2 : \overrightarrow{C\setminus \{p\}} \succ d \succ p \succ q\]

 Let $s(\cdot)$ and $s^{\prime}(\cdot)$ be the score functions for the \textsc{Coalitional-manipulation} and the \textsc{Uniform-frugal-\$bribery} instances respectively. We assume that $s(x) < s(p) + 2(m-1)$ for every $x\in C\setminus\{p\}$ since otherwise the \textsc{Coalitional-manipulation} instance is a trivial \NO instance. We now add more votes to ensure following score differences in the resulting \textsc{Uniform-frugal-\$bribery} instance. 
 
 $$s^{\prime}(p) = \lambda + s(p) - 2, s^{\prime}(x) = \lambda + s(x) \text{ for every } x\in C\setminus\{p\},$$ 
 
 $$s^{\prime}(q) = s^{\prime}(p) + 2m - 1, s^{\prime}(p) > s^{\prime}(d) + 2m \text{ for some } \lambda \in \mathbb{Z}$$ 
 
 This will be achieved as follows. For any two arbitrary candidates $a$ and $b$, the following two votes increase the score of $a$ by one more than the rest of the candidates except $b$ whose score increases by one less. This construction has been used before~\cite{xia2010scheduling,davies2011complexity}.
 
 \[ a \succ b \succ \overrightarrow{C\setminus \{a,b\}} \]
 \[ \overleftarrow{C\setminus \{a,b\}} \succ a \succ b \]
 
 Also, we can ensure that candidate $p$ is always in $(\nfrac{m-1}{2}, \nfrac{m+1}{2})$ positions and the candidate $q$ never {\em immediately} follows $p$ in these new votes. The tie-breaking rule is ``$q \suc \text{others} \succ p$.'' The candidate $q$ is the winner since $s^{\prime}(q) = s^{\prime}(p) + 2m - 1 \ge s^\pr(x)$ for every $x\in C^\pr\setminus\{p,q\}.$ The distinguished candidate is $p$. The price of every vulnerable vote is one and the budget is two. We claim that the two instances are equivalent. 
 
 In the forward direction, suppose the \textsc{Coalitional-manipulation} instance is a \YES{} instance. Let $u_1, u_2$ be the manipulators' votes that make $p$ win. In the \textsc{Frugal-\$bribery} instance, we replace $\nu_i$ by $\nu_i^{\prime} : p \succ d \succ (u_i\setminus\{p\}) \succ q$ for $i=1,2.$ This makes $p$ win the election. 
 
 In the reverse direction, recall that in all the vulnerable votes except $\nu_1$ and $\nu_2$, the candidate $q$ never {\em immediately} follows candidate $p$. Therefore, changing any of these votes can never make $p$ win the election since $s^{\prime}(q) = s^{\prime}(p) - 2m + 1$ and the budget is two. Hence, the only way $p$ can win the election, if at all possible, is by changing the votes $\nu_1$ and $\nu_2$. Let a vote $\nu_i^{\prime}$ replaces $\nu_i$ for $i=1,2$. We can assume, without loss of generality, that $p$ and $d$ are at the first and the second positions respectively in both $\nu_1^{\prime}$ and $\nu_2^{\prime}$. Let $u_i$ be the order $\nu_i^{\prime}$ restricted only to the candidates in $C$. This makes $p$ the unique winner of the \textsc{Coalitional-manipulation} instance since $s^{\prime}(p) = \lambda + s(p) - 2, s^{\prime}(x) = \lambda + s(x)$ for every $x\in C$ and the tie-breaking rule is ``$q \suc \text{others} \succ p$.''
\end{proof}}

\section{Results for Weighted Elections}\label{sec:wt}

\aaaiversion{We divide Observation $2$ of short version into Observation $4$ and $5$ below.}
Now we turn our attention to weighted elections. As before, we begin with some easy observations that follow from known results. \shortversion{
The first part of \Cref{lem:wfrugal} follows from the literature on the \textsc{Coalitional-manipulation} problem and~\Cref{prop:man_bri_con} whereas the second part of \Cref{lem:wfrugal} follows from the proof of Theorem $6$ in~\cite{conitzer2007elections}.

\begin{observation}\shortversion{[$\star$]}\label{lem:wfrugal} 
 \begin{itemize}[noitemsep,leftmargin=*]
  \item The \textsc{Frugal-bribery} problem is in \Pshort{} for the maximin and Copeland voting rules for three candidates.
  \item The \textsc{Frugal-bribery} problem is \NPC{} for any scoring rule except plurality for three candidates.
 \end{itemize}
\end{observation}
}

\longversion{

\begin{observation}\label{lem:wfrugalEasy}
 The \textsc{Frugal-bribery} problem is in \Pshort{} for the maximin and the Copeland voting rules for three candidates.
\end{observation}

\begin{proof}
 When we have $3$ candidates, the \textsc{Coalitional Manipulation} problem is in \Pshort{} for the maximin and the Copeland voting rules~\cite{conitzer2007elections}. Hence, the result follows from \Cref{prop:man_bri_con}.
\end{proof}

Using the proof of Theorem 6 in Conitzer et al.~\cite{conitzer2007elections}, we can obtain the following.

\begin{observation}\label{lem:wfrugalScr}
 Assume we have only $3$ candidates. Then the \textsc{Frugal-bribery} problem is \NPC{} for every scoring rule except plurality.
\end{observation}
}

\begin{theorem}\shortversion{[$\star$]}\label{thm:wfrugalP} 
 The \textsc{Frugal-bribery} problem is in \Pshort{} for the plurality voting rule.
\end{theorem} 

\longversion{ \begin{proof}
 Let $p$ be the distinguished candidate of the campaigner. We greedily replace every vulnerable vote by the vote $p \succ others$. The correctness follows from the fact that plurality only accounts for candidates in the top position, and the strategy described is therefore the best possible for candidate $p$. 
\end{proof}}

Our hardness results in this section are based on the \textsc{Partition} problem, which is known to be \NPC{} \cite{garey1979computers}, and is defined as follows.

\defproblem{\textsc{Partition}}{A finite multi-set $W$ of positive integers with $\sum_{w\in W} w = 2K$.}{Does there exist a subset $W^{\prime} \subset W$ such that $\sum_{w\in W^{\prime}} w = K$?}


 An arbitrary instance of \textsc{Partition} is denoted by $(W,2K)$.
 We define another problem which we call $\frac{1}{4}$-\textsc{Partition} as below. \longversion{We prove that $\frac{1}{4}$-\textsc{Partition} is also \NPC{}}\shortversion{The $\frac{1}{4}$-\textsc{Partition} problem can be proved to be \NPC{} by reducing from the \textsc{Partition} problem}.
We will use this fact in the proof of \Cref{thm:stv_wt}.

\defproblem{$\frac{1}{4}$-\textsc{Partition}}{A finite multi-set $W$ of positive integers with $\sum_{w\in W} w = 4K$.}{Does there exist a subset $W^{\prime} \subset W$ such that $\sum_{w\in W^{\prime}} w = K$?}


 An arbitrary instance of $\frac{1}{4}$-\textsc{Partition} is denoted by $(W,4K)$. 
\longversion{\begin{lemma}\label{lem:part}
$\frac{1}{4}$-\textsc{Partition} problem is \NPC{}.
\end{lemma}

\begin{proof}
 The problem is clearly in \NPshort{}. To show \NPshort{}-hardness, we reduce the \textsc{Partition} problem to it. 
 Let $(W,2K)$ be an arbitrary instance of the \textsc{Partition} problem. We can assume, without loss of generality, that $2K \notin W$, since 
 otherwise the instance is trivially a \textit{no} instance. 
 The corresponding $\frac{1}{4}$-\textsc{Partition} problem instance is defined by $(W_1,4K)$, where $W_1 = W \cup \{2K\}$. 
 We claim that the two instances are equivalent. Suppose the \textsc{Partition} instance is a \YES{} instance and thus 
 there exists a set $W^{\prime}\subset W$ such that $\sum_{w\in W^{\prime}} w = K$. This $W^{\prime}$ gives a solution to 
 the $\frac{1}{4}$-\textsc{Partition} instance. To prove the result in the other direction, 
 suppose there is a set $W^{\prime}\subset W_1$ such that $\sum_{w\in W^{\prime}} w = K$. 
 This $W^{\prime}$ gives a solution to the \textsc{Partition} problem instance since $2K \notin W^{\prime}$.
\end{proof}}

Our hardness reductions in this section have the following overall approach. We first introduce vulnerable votes corresponding to the numbers in the instance of \textsc{Partition}, and the weights and prices of these votes are tightly correlated with the corresponding numbers in the \textsc{Partition} instance. We then introduce auxiliary votes, that are typically not vulnerable, but are crafted in such a way that the distinguished candidate lags behind the current winner --- and the differential can only be compensated by changing \textit{exactly} half the weight of the vulnerable votes. In the case of \textsc{Frugal-\$bribery}, this is relatively easy to achieve: the score difference can be used to create the requirement that the total weight of the changed votes is at least $K$, while the budget can be used to enforce that the total cost of the changed votes is at most $K$, which leads us naturally to the desired partition. We see this in play in~\Cref{thm:wfrugalPluNPC}, for the plurality voting rule. For all the other rules, since we don't have costs, a more delicate argument is required to enforce the two-sided dynamic of the affected votes.

\longversion{In the rest of this section, we present the hardness results in weighted elections for the following voting rules: plurality, maximin, STV, Copeland$^\alpha$, and simplified Bucklin. For plurality, recall that the \textsc{Frugal-bribery} problem is in \Pshort{}, and we will show that \textsc{Frugal-\$bribery} is \NPC{}. For all the other rules, we will establish that even \textsc{Frugal-bribery} is \NPC{}.} \shortversion{ The following result can be proved by exhibiting a reduction from the \textsc{Partition} problem.}

\begin{theorem}\shortversion{[$\star$]}\label{thm:wfrugalPluNPC}
 The \textsc{Frugal-\$bribery} problem is \NPC{} for the plurality voting rule for three candidates.
\end{theorem}

\longversion{\begin{proof}
 The problem is clearly in \NPshort{}. We reduce an arbitrary instance of \textsc{Partition} to an instance of \textsc{Frugal-\$bribery} for the plurality voting rule. Let $(W,2K),$ with $W=\{w_1, \ldots, w_n\}$ and $\sum_{i=1}^n w_i=2K$, be an arbitrary instance of the \textsc{Partition} problem. In the reduced instance, we introduce three candidates, namely, $p, a,$ and $b$. The distinguished candidate is $p$. We will now add votes in such a way that makes $b$ win the election.  

 For every $i\in [n]$, we have one vote $a\succ p\succ b$ of both weight and price $w_i$. We have two votes $b\succ p\succ a$ of weight $3K$ each (we do not need to define the price of this vote since it is non-vulnerable). We also have one vote $p\succ a\succ b$ of both weight and price $2K+1$. This finishes the description of the votes. We observe that candidate $b$ wins the plurality election with plurality score $3K$. The tie-breaking rule is ``$a\succ b\succ p$.'' We define the budget to be $K$. We claim that the two instances are equivalent. 
 
 In the forward direction, suppose there exists a $W^{\prime} \subset W$ such that $\sum_{w\in W^{\prime}} w = K$. We change the votes corresponding to the weights in $W^{\prime}$ to $p\succ a\succ b$. This makes $p$ win the election with a plurality score of $3K+1$. 
 
 To prove the other direction, for $p$ to win, its score must increase by at least $K$. Also, the prices ensure that $p$'s score can increase by at most $K$. Hence, $p$'s score must increase by exactly by $K$ and the only way to achieve this is to increase its score by changing the votes corresponding to the weights in $W$. Thus, $p$ can win only if there exists a $W^{\prime} \subset W$ such that $\sum_{w\in W^{\prime}} w = K$.
\end{proof}}

 Next we show the hardness result for the maximin\longversion{ voting} rule.
 
\begin{theorem}\label{thm:wfrugalMaxmin}
 The \textsc{Frugal-bribery} problem is \NPC{} for the maximin voting rule for $4$ candidates.
\end{theorem}

\begin{proof}
 The problem is clearly in \NPshort{}. We reduce an arbitrary instance of \textsc{Partition} to an instance of \textsc{Frugal-bribery} for the maximin voting rule. Let $(W,2K),$ with $W=\{w_1, \ldots, w_n\}$ and $\sum_{i=1}^n w_i = 2K$, be an arbitrary instance of the \textsc{Partition} problem. The candidates in our \textsc{Frugal-bribery} instance are $p, a, b,$ and $c$. For every $i\in [n]$, we have one vote $p\succ a\succ b\succ c$ of weight $w_i$. There is one vote $c\succ a\succ b\succ p$, one vote $b\succ c\succ a\succ p$, and one vote $a\succ c\succ b\succ p$ each of weight $K$. The election is summarized in~\Cref{tab:maximin}. This finishes the description of the votes. The weighted majority graph induced by these votes are shown in \Cref{fig:maximin_initial}. We observe that candidate $a$ wins since it is the Condorcet winner of the election. The tie-breaking rule is ``$p\succ a\succ b\succ c$.'' The distinguished candidate is $p$. Let $T$ denote the set of votes corresponding to the weights in $W$ and the rest of the votes $S$. Notice that only the votes in $T$ are vulnerable. We claim that the two instances are equivalent. 
 
\begin{table}[H]
\centering
\vspace{10pt}
\begin{minipage}{.5\textwidth}
\begin{tabular}{|c|c|c|c|c|}
\hline
Pairwise Outcomes & $a$                          & $b$                                                 & $c$                                                 & $p$                                                 \\ \hline
$a$               & ---                          & $4K$                                                & \cellcolor[HTML]{9AFF99}{\color[HTML]{000000} $3K$} & \cellcolor[HTML]{9AFF99}{\color[HTML]{000000} $3K$} \\ \hline
$b$               & \cellcolor[HTML]{9AFF99}$K$  & ---                                                 & $3K$                                                & $3K$ \\ \hline
$c$               & \cellcolor[HTML]{9AFF99}$2K$ & \cellcolor[HTML]{9AFF99}{\color[HTML]{000000} $2K$} & ---                                                 & $3K$ \\ \hline
$p$               & \cellcolor[HTML]{9AFF99}$2K$ & \cellcolor[HTML]{9AFF99}$2K$                        & \cellcolor[HTML]{9AFF99}$2K$                        & ---                                                 \\ \hline
\end{tabular}

\end{minipage}\hfill
\begin{minipage}{.5\textwidth}
\begin{tabular}{|c|c|c|c|c|}
\hline
Pairwise Outcomes & $a$                          & $b$                                                 & $c$                                                 & $p$                                                 \\ \hline
$a$               & ---                          & $3K$                                                & \cellcolor[HTML]{9AFF99}{\color[HTML]{000000} $2K$} & $3K$ \\ \hline
$b$               & \cellcolor[HTML]{9AFF99}$2K$  & ---                                                 & $3K$                                                & $3K$ \\ \hline
$c$               & $3K$ & \cellcolor[HTML]{9AFF99}{\color[HTML]{000000} $2K$} & ---                                                 & $3K$ \\ \hline
$p$               & \cellcolor[HTML]{9AFF99}$2K$ & \cellcolor[HTML]{9AFF99}$2K$                        & \cellcolor[HTML]{9AFF99}$2K$                        & ---                                                 \\ \hline
\end{tabular}
\end{minipage}\hfill
\caption{Every cell shows the number of voters who prefer the row candidate over the column candidate. The left table shows the case of the reduced election in \Cref{thm:wfrugalMaxmin}. The right table shows the case of the modified election in the forward direction of the proof of \Cref{thm:wfrugalMaxmin}. The green cells show the witness of worst-case pairwise elections for every row candidate.}
\label{tab:maximin}
\end{table}

 \begin{figure}[htbp]
 \begin{center}
 \begin{tikzpicture}[scale=2]
  \node[draw,circle] (a) at (2,1) {a};
  \node[draw,circle] (b) at (2,0) {b};
  \node[draw,circle] (c) at (0,0) {c};
  \node[draw,circle] (p) at (0,1) {p};
  
  \draw[->] (a) -- node[above] {K} (p);
  \draw[->] (b) -- node[above] {K} (p);
  \draw[->] (c) -- node[left] {K} (p);
  \draw[->] (a) -- node[below] {K} (c);
  \draw[->] (b) -- node[below] {K} (c);
  \draw[->] (a) -- node[right] {3K} (b);
 \end{tikzpicture}
 \end{center}
 \caption{Weighted majority graph of the reduced instance in \Cref{thm:wfrugalMaxmin}.}\label{fig:maximin_initial}
 \end{figure}
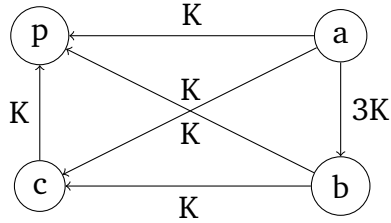

 In the forward direction, suppose there exists a $W^{\prime} \subset W$ such that $\sum_{w\in W^{\prime}} w = K$. We keep the votes corresponding to the weights in $W^{\prime}$ same as the original vote $p\succ a\succ b\succ c$. We change the rest of the votes in $T$ to $p\succ b\succ c\succ a$. We observe from the weighted majority graph (shown in \Cref{fig:maximin_forward}) induced by these new set of votes that the maximin score of every candidate is $-K$ and thus due to the tie-breaking rule, $p$ wins\longversion{ the election}.
 
 \begin{figure}[htbp]
 \begin{center}
 \begin{tikzpicture}[scale=2]
  \node[draw,circle] (a) at (2,1) {a};
  \node[draw,circle] (b) at (2,0) {b};
  \node[draw,circle] (c) at (0,0) {c};
  \node[draw,circle] (p) at (0,1) {p};
  
  \draw[->] (a) -- node[above] {K} (p);
  \draw[->] (b) -- node[above] {K} (p);
  \draw[->] (c) -- node[left] {K} (p);
  \draw[->] (c) -- node[below] {K} (a);
  \draw[->] (b) -- node[below] {K} (c);
  \draw[->] (a) -- node[right] {K} (b);
 \end{tikzpicture}
 \end{center}
 \caption{Weighted majority graph induced by the votes in the forward direction of the proof of \Cref{thm:wfrugalMaxmin}.}\label{fig:maximin_forward}
 \end{figure}
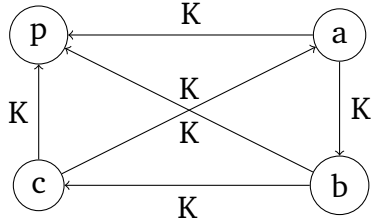

 To prove the result in the other direction, suppose there is a way to change the vulnerable votes, that is the votes in $T$, that makes $p$ win the election. Let the new set of votes that replace $T$ be $T^\pr$. Without loss of generality, we can assume that all the votes in $T^\pr$ place $p$ at the first position. Hence the maximin score of candidate $p$ is $-K$. We first notice that the only way $p$ could win is that the vertices $a, b,$ and $c$ must form a cycle in the weighted majority graph. Otherwise, one of $a, b,$ and $c$ will be\longversion{ a Condorcet winner and thus} the winner of the election. We also observe that the weight of every edge of the cycle consisting of $a, b,$ and $c$ only must be at least $K$. If not, then the maximin score of one of candidates in $\{a, b, c\}$ must be strictly more than $-K$. This contradicts our assumption that $p$ wins since the maximin score of candidate $p$ is fixed at $-K$. 
 
 Now, we claim that candidate $b$ must defeat candidate $c$. Suppose not. Then, since the maximin score of $p$ is fixed at $-K$, $c$ must defeat $b$ by a margin of at least $K$. Further, note that the margin must be exactly $K$ since making $c$ defeat $b$ by a margin $K$ requires $c$ to be preferred over $b$ in every vote in $T^\pr$. On the other hand, $a$ must defeat $c$ by a margin of at least $K$ (and thus exactly $K$), otherwise the maximin score of $c$ will be more than $-K$. This implies that all the votes in $T^\pr$ must be $p\succ a\succ c\succ b$ which makes $a$ defeat $b$. This is a contradiction since the vertices $a, b,$ and $c$ must form a cycle in the weighted majority graph. Hence $b$ must defeat $c$ by a margin of $K$. 
 
 Since $b$ defeats $c$ by a margin of $K$, every vote in $T^\pr$ is forced to prefer $b$ over $c$. Without loss of generality, we assume that all the votes in $T^\pr$ are either $p\succ a\succ b\succ c$ or $p\succ b\succ c\succ a$, since whenever $c$ is immediately after $a$, we can swap $a$ and $c$ and this will only reduce the score of $a$ without affecting the score of any other candidates. If the total weight of the votes $p\succ a\succ b\succ c$ in $T^\pr$ is more than $K$, then $D_E(c,a) < K$, thereby making the maximin score of $a$ more than the maximin score of $p$. If the total weight of the votes $p\succ a\succ b\succ c$ in $T^\pr$ is less than $K$, then $D_E(a,b) < K$, thereby making the maximin score of $b$ more than the maximin score of $p$. Thus the total weight of the votes $p\succ a\succ b\succ c$ in $T^\pr$ should be exactly $K$ which corresponds to a partition of $W$.
\end{proof}

 We now prove the hardness result for the STV voting rule.
 
\begin{theorem}\label{thm:stv_wt}
 The \textsc{Frugal-bribery} problem is \NPC{} for the STV voting rule for $3$ candidates.
\end{theorem}

\begin{proof}
 The problem is clearly in \NPshort{}. We reduce an arbitrary instance of $\frac{1}{4}$-\textsc{Partition} to an instance of \textsc{Frugal-bribery} for the STV voting rule. Let $(W,4K),$ with $W=\{w_1, \ldots, w_n\}$ and $\sum_{i=1}^n w_i = 4K$, be an arbitrary instance of the $\frac{1}{4}$-\textsc{Partition} problem. The candidates in our \textsc{Frugal-bribery} instance are $p, a,$ and $b$. For every $i\in [n]$, we have a vote $p\succ a\succ b$ of weight $w_i$. We have one vote $a\succ p\succ b$ of weight $3K-1$ and one vote $b\succ a\succ p$ of weight $2K$. This finishes the description of the votes. The tie-breaking rule is ``$a\succ b\succ p$.'' We observe that the plurality score of candidates $a, b,$ and $p$ are $3K-1, 2K,$ and $4K$ respectively in the resulting election. Hence candidate $b$ gets eliminated in the first round. In the second round, the plurality score of candidates $a$ and $p$ are $5K-1$ and $4K$ respectively. Hence candidate $a$ wins the STV election. The distinguished candidate is $p$. Let $T$ denote the set of votes corresponding to the weights in $W$ and the rest of the votes be $S$. Notice that only the votes in $T$ are vulnerable. We claim that the two instances are equivalent.
 
 In the forward direction, suppose there exists a $W^{\prime} \subset W$ such that $\sum_{w\in W^{\prime}} w = K$. We change the votes corresponding to the  weights in $W^{\prime}$ to $b\succ p\succ a$. We do not change the rest of the votes in $T$. In the first round of the resulting profile, candidates $a, b,$ and $p$ receive a plurality score of $3K-1, 3K,$ and $3K$ respectively. Hence candidate $a$ gets eliminated in the first round. In the second round, candidates $b$ and $p$ receive a plurality score of $3K$ and $6K-1$ respectively and thus candidate $p$ wins the election.
 
 
For the other direction, suppose there is a way to change the votes in $T$ that makes $p$ win the election. We first observe that candidate $p$ can win only if $p$ and $b$ qualifies for the second round. Hence, the total weight of the votes in $T$ that put $b$ at the first position must be at least $K$. On the other hand, if the total weight of the votes in $T$ that put $b$ at the first position is strictly more than $K$, then $p$ does not qualify for the second round and thus cannot win the election. Hence the total weight of the votes in $T$ that put $b$ at the first position must be exactly equal to $K$ which constitutes a $\frac{1}{4}$-partition of $W$.
\end{proof}

 For three candidates, the STV voting rule is the same as the plurality with runoff voting rule. Hence, we have the following corollary.
 
\begin{corollary}\label{cor:run_wt}
 The \textsc{Frugal-bribery} problem is \NPC{} for the plurality with runoff voting rule for $3$ candidates.
\end{corollary}

\longversion{
We turn our attention to the Copeland$^{\alpha}$ voting rule next.

\begin{theorem}\label{thm:wfrugalCopeland}
 The \textsc{Frugal-bribery} problem is \NPC{} for the Copeland$^{\alpha}$ voting rule for $4$ candidates for every $\alpha\in[0,1)$.
\end{theorem}

\begin{proof}
 The problem is clearly in \NPshort{}. We reduce an arbitrary instance of \textsc{Partition} to an instance of \textsc{Frugal-bribery} for the Copeland$^{\alpha}$ voting rule. Let $(W,2K),$ with $W=\{w_1, \ldots, w_n\}$ and $\sum_{i=1}^n w_i = 2K$, be an arbitrary instance of the \textsc{Partition} problem. The candidates in our \textsc{Frugal-bribery} instance are $p, a, b,$ and $c$. For every $i\in [n]$, we have a vote $p\succ a\succ b\succ c$ of weight $w_i$. There are two votes $a\succ p\succ b\succ c$ and $c\succ b\succ a\succ p$ each of weight $K+1$. This finishes the description of the votes.  The tie-breaking rule is ``$a\succ b\succ c\succ p$.'' The weighted majority graph induced by these votes are shown in \Cref{fig:copeland_initial}. We observe that candidate $a$ wins since it is the Condorcet winner of the election. The distinguished candidate is $p$. Let $T$ denote the set of votes corresponding to the weights in $W$ and the rest of the votes be $S$. Notice that only the votes in $T$ are vulnerable. We claim that the two instances are equivalent.

\begin{table}[H]
\centering
\vspace{10pt}
\begin{minipage}{.5\textwidth}
\resizebox{\textwidth}{!}{
\begin{tabular}{|c|c|c|c|c|}
\hline
\makecell{Pairwise\\Outcomes} & $a$                          & $b$                                                 & $c$                                                 & $p$                                                 \\ \hline
$a$               & ---                          & \cellcolor[HTML]{9AFF99}$3K+1$                                                & \cellcolor[HTML]{9AFF99}$3K+1$ & \cellcolor[HTML]{9AFF99}$2K+2$ \\ \hline
$b$               & $K+1$  & ---                                                 & \cellcolor[HTML]{9AFF99}$3K+1$                                                & $K+1$ \\ \hline
$c$               & $K+1$ & $K+1$ & ---                                                 & $K+1$ \\ \hline
$p$               & $2K$ & \cellcolor[HTML]{9AFF99}$3K+1$                        & \cellcolor[HTML]{9AFF99}$3K+1$                        & ---                                                 \\ \hline
\end{tabular}
}
\end{minipage}\hfill
\begin{minipage}{.5\textwidth}
\resizebox{\textwidth}{!}{
\begin{tabular}{|c|c|c|c|c|}
\hline
\makecell{Pairwise\\Outcomes} & $a$                          & $b$                                                 & $c$                                                 & $p$                                                 \\ \hline
$a$               & ---                          & $K+1$                                                & $K+1$ & \cellcolor[HTML]{9AFF99}$2K+2$ \\ \hline
$b$               & \cellcolor[HTML]{9AFF99}$3K+1$  & ---                                                 & \cellcolor{arylideyellow}$2K+1$                                                & $K+1$ \\ \hline
$c$               & \cellcolor[HTML]{9AFF99}$3K+1$ & \cellcolor{arylideyellow}$2K+1$ & ---                                                 & $K+1$ \\ \hline
$p$               & $2K$ & \cellcolor[HTML]{9AFF99}$3K+1$                        & \cellcolor[HTML]{9AFF99}$3K+1$                        & ---                                                 \\ \hline
\end{tabular}
}
\end{minipage}\hfill
\caption{Every cell shows the number of voters who prefer the row candidate over the column candidate. The left table shows the case of the reduced election in \Cref{thm:wfrugalCopeland}. The right table shows the case of the modified election in the forward direction of the proof of \Cref{thm:wfrugalCopeland}. Green cells indicate that the row candidate defeats the column candidate in pairwise election. Yellow cells indicate that the row and the column candidate are tied in pairwise election.}
\label{tab:copeland}
\end{table} 
 
 \begin{figure}[!htbp]
 \begin{center}
 \begin{tikzpicture}[scale=2]
  \node[draw,circle] (a) at (2,1) {a};
  \node[draw,circle] (b) at (2,0) {b};
  \node[draw,circle] (c) at (0,0) {c};
  \node[draw,circle] (p) at (0,1) {p};
  
  \draw[->] (a) -- node[above] {2} (p);
  \draw[->] (p) -- node[above] {2K} (b);
  \draw[->] (p) -- node[left] {2K} (c);
  \draw[->] (a) -- node[below] {2K} (c);
  \draw[->] (b) -- node[below] {2K} (c);
  \draw[->] (a) -- node[right] {2K} (b);
 \end{tikzpicture}
 \end{center}
 \caption{Weighted majority graph of the reduced instance in \Cref{thm:wfrugalCopeland}.}\label{fig:copeland_initial}
 \end{figure}
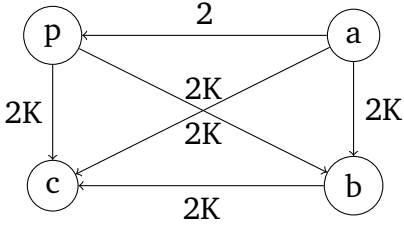
 
 In the forward direction, suppose there exists a $W^{\prime} \subset W$ such that $\sum_{w\in W^{\prime}} w = K$. We change the votes corresponding to the weights in $W^{\prime}$ to $p\succ c\succ b\succ a$. We change the rest of the votes in $T$ to $p\succ b\succ c\succ a$. We observe from the weighted majority graph (shown in \Cref{fig:copeland_forward}) induced by these new set of votes that the Copeland$^{\alpha}$ score of of candidate $p$ is $2$ and the Copeland$^{\alpha}$ score of every other candidate is strictly less than $2$. Hence, candidate $p$ wins the election.
 
 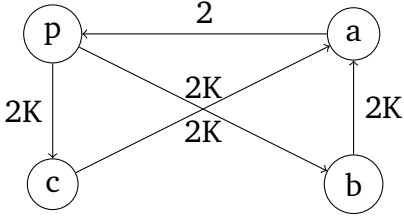
\begin{figure}[htbp]
 \begin{center}
 \begin{tikzpicture}[scale=2]
  \node[draw,circle] (a) at (2,1) {a};
  \node[draw,circle] (b) at (2,0) {b};
  \node[draw,circle] (c) at (0,0) {c};
  \node[draw,circle] (p) at (0,1) {p};
  
  \draw[->] (a) -- node[above] {2} (p);
  \draw[->] (p) -- node[above] {2K} (b);
  \draw[->] (p) -- node[left] {2K} (c);
  \draw[->] (c) -- node[below] {2K} (a);
  \draw[->] (b) -- node[right] {2K} (a);
 \end{tikzpicture}
 \end{center}
 \caption{Weighted majority graph induced by the votes in the forward direction of the proof of \Cref{thm:wfrugalCopeland}.}\label{fig:copeland_forward}
 \end{figure}
 
 For the other direction, suppose there is a way to change the votes in $T$ that makes $p$ win the election. Let the new set of votes that replace $T$ be $T^\pr$. Without loss of generality, we can assume that all the votes in $T^\pr$ place $p$ at the top position. We claim that one of the three pairwise elections among $a, b,$ and $c$ must be a tie. Suppose not, then $a$ must lose to both $b$ and $c$, otherwise $a$ wins the election due to the tie-breaking rule. Now consider the pairwise election between $b$ and $c$. If $b$ defeats $c$, then $b$ wins the election due to the tie-breaking rule. If $c$ defeats $b$, then $c$ wins the election again due to the tie-breaking rule. Hence, one of the pairwise elections among $a, b,$ and $c$ must be a tie. Without loss of generality suppose $a$ and $b$ ties. However, then the total weight of the votes that prefer $a$ to $b$ in $T^\pr$ must be $K$ which constitutes a partition of $W$.
\end{proof}}

\longversion{
Finally, we show that the \textsc{Frugal-bribery} problem for the simplified Bucklin voting rule is \NPC{}.

\begin{theorem}\label{thm:bucklin_wt}
 The \textsc{Frugal-bribery} problem is \NPC{} for the simplified Bucklin voting rule for $4$ candidates.
\end{theorem}

\begin{proof}
 The problem is clearly in \NPshort{}. We reduce an arbitrary instance of \textsc{Partition} to an instance of \textsc{Frugal-bribery} for the simplified Bucklin voting rule. Let $(W,2K),$ with $W=\{w_1, \ldots, w_n\}$ and $\sum_{i=1}^n w_i = 2K$, be an arbitrary instance of the \textsc{Partition} problem. The candidates in our \textsc{Frugal-bribery} instance are $p, a, b,$ and $c$. For every $i\in [n]$, we have one vote $p\succ a\succ b\succ c$ of weight $w_i$. There are two votes $a\succ b\succ p\succ c$ and $c\succ b\succ a\succ p$ each of weight $K$. This finishes the description of the votes. The tie-breaking rule is ``$p\succ a\succ b\succ c$.'' Observe that:

\begin{itemize}
\item The candidates $a$ and $b$ get majority within the first two positions	.
\item The candidate $p$ does \textit{not} get majority within the first two positions.
\item No candidate gets majority within the first position.
\end{itemize}

Therefore, candidate $a$ wins due to the tie-breaking rule. We set the distinguished candidate as $p$. Let $T$ denote the set of votes corresponding to the weights in $W$ and the rest of the votes be $S$. Notice that only the votes in $T$ are vulnerable. We claim that the two instances are equivalent.
 
 In the forward direction, suppose there exists a $W^{\prime} \subset W$ such that $\sum_{w\in W^{\prime}} w = K$. We change the votes corresponding to the weights in $W^{\prime}$ to $p\succ c\succ b\succ a$. We keep the votes corresponding to the weights in $W\setminus W^{\prime}$ same as the original ones. Now no candidate gets majority within first two positions and candidate $p$ gets majority within first two positions. This makes $p$ win the election with a simplified Bucklin score of $3$ due to the tie-breaking rule.
 
 
To prove the result in the other direction, suppose there is a way to change the votes in $T$ that makes $p$ win the election. Let the new set of votes that replace $T$ be $T^\pr$. Without loss of generality, we can assume that all the votes in $T^\pr$ place $p$ at the first position. We first notice that the simplified Bucklin score of $p$ is already fixed at three. In the votes in $T^\pr$, candidate $b$ can never be placed at the second position since that will make the simplified Bucklin score of $b$ to be two. Also the total weight of the votes in $T^\pr$ that place $a$ in their second position can be at most $K$. The same holds for $c$. Hence, the total weight of the votes that place $a$ in their second position will be exactly equal to $K$ which constitutes a partition of $W$.
\end{proof}}

\shortversion{ We also have the following results for the Copeland$^\alpha$ and simplified Bucklin voting rules by reducing from \textsc{Partition}.
\begin{theorem}\shortversion{[$\star$]}\label{thm:wfrugalCopeland}
 The \textsc{Frugal-bribery} problem is \NPC{} for the Copeland$^{\alpha}$ and simplified Bucklin voting rules for four candidates, whenever $\alpha\in[0,1)$.
\end{theorem}
}

 From \Cref{prop:conn}, \Cref{\longversion{lem:wfrugalScr}\shortversion{lem:wfrugal}}, \Cref{thm:wfrugalPluNPC,thm:wfrugalMaxmin,thm:wfrugalCopeland,thm:stv_wt\longversion{,thm:bucklin_wt}}, 
and \Cref{cor:run_wt}, we get the following.

\begin{corollary}
 When we have $3$ candidates, the \textsc{Uniform-frugal-\$bribery} and the \textsc{Nonuniform-frugal-\$bribery} problems are \NPC{} for the scoring rules except 
 plurality, STV, and the plurality with runoff voting rules. When we have $4$ candidates, the \textsc{Uniform-frugal-\$bribery} and the \textsc{Nonuniform-frugal-\$bribery} problems are \NPC{} for the maximin, Copeland, and simplified Bucklin voting rules.
\end{corollary}
\section{Conclusion and Future Work}\label{sec:con}

We have proposed and studied two important special cases of the \textsc{\$Bribery} problem where the briber is frugal. We have shown that the {\sc Frugal-bribery} problem can sometimes be polynomial time solvable even if the {\sc \$Bribery} and the {\sc Swap-bribery} problems are \NPC as observed for the $k$-approval and the $k$-veto voting rules for unweighted elections. This establishes success in finding important practical special cases of the sophisticated {\sc \$Bribery} and {\sc Swap-bribery} problems. We also proved that the {\sc Frugal-bribery} problem is \NPC for the Borda voting rule and the {\sc Frugal-\$bribery} problem is \NPC for all the voting rules studied here except the plurality and the veto voting rules for unweighted elections. The intractability results of the {\sc Frugal-\$bribery} problem and the {\sc Frugal-\$bribery} problem thereby subsumes and strengthens the hardness results for the {\sc \$Bribery} problem. For the weighted election, we have shown that the simplest {\sc Frugal-bribery} problem also is \NPC for all the voting rule studied in this paper except for the plurality voting rule even when the number of candidates is as small as $3$ or $4$. We find these results in the weighted elections both surprising and interesting.


An immediate future work is to resolve the open cases in \Cref{tbl:frugal_summary}. Another important direction for future work is to study these problems under various other settings. Notably, one might consider enhancing our proposed model further to account for constraints that arise in practical scenarios. For instance, we might want to restrict the campaigner's knowledge about the votes and/or the candidates who will actually turn up. The uncertainty can also arise from the voting rule that will eventually be used among a set of voting rules. Also, studying these problems when the pricing model for vulnerable votes is similar to swap bribery would be another interesting future direction. We believe that a game theoretic perspective of the problem may also yield valuable insights. 

\subsubsection*{Acknowledgement} Palash Dey wishes to gratefully acknowledge support from Google India for providing him with a special fellowship for carrying out his doctoral work. Neeldhara Misra acknowledges support by the INSPIRE Faculty Scheme, DST India (project IFA12-ENG-31).

\section*{References}
\longversion{\bibliographystyle{alpha}
\bibliography{frugal}}

\end{document}